\documentclass{article}

\usepackage{microtype}
\usepackage{graphicx}

\usepackage{booktabs} 

\usepackage{hyperref}

\usepackage[accepted]{icml2026}

\usepackage{amsmath}
\usepackage{amssymb}
\usepackage{amsthm}
\usepackage{mathtools}
\usepackage{booktabs}

\usepackage{multirow}
\usepackage{makecell} 
\usepackage{tabularx}
\usepackage[capitalize,noabbrev]{cleveref}
\usepackage{pifont} 

\theoremstyle{plain}
\newtheorem{theorem}{Theorem}[section]
\newtheorem{proposition}[theorem]{Proposition}
\newtheorem{lemma}[theorem]{Lemma}
\newtheorem{corollary}[theorem]{Corollary}
\theoremstyle{definition}
\newtheorem{definition}[theorem]{Definition}
\newtheorem{assumption}[theorem]{Assumption}
\theoremstyle{remark}

\usepackage[textsize=tiny]{todonotes}

\usepackage{amsmath,amssymb,amsthm}
\usepackage{algorithm,algorithmic}
\usepackage{hyperref}
\usepackage{cleveref}
\usepackage{booktabs}
\usepackage{graphicx}
\usepackage{subcaption}
\usepackage{enumitem}

\usepackage{subcaption}

\usepackage[table]{xcolor} 
\usepackage{colortbl}
\usepackage[table,xcdraw]{xcolor}

\newcommand{\EE}{\mathbb{E}}

\definecolor{DeepCyan}{RGB}{38, 186, 196}

\definecolor{lxs}{RGB}{200,0,0}

\icmltitlerunning{Understanding Agent Scaling in LLM-Based Multi-Agent Systems via Diversity}

\begin{document}

\twocolumn[







\icmltitle{Understanding Agent Scaling in LLM-Based Multi-Agent Systems via Diversity}





\begin{icmlauthorlist}
\icmlauthor{Yingxuan Yang}{yyy}
\icmlauthor{Chengrui Qu}{sch}
\icmlauthor{Muning Wen}{yyy}
\icmlauthor{Laixi Shi}{xxx}
\icmlauthor{Ying Wen}{yyy}
\icmlauthor{Weinan Zhang}{yyy}
\icmlauthor{Adam Wierman}{sch}
\icmlauthor{Shangding Gu*}{comp}
\end{icmlauthorlist}

\icmlaffiliation{yyy}{Shanghai Jiao Tong University,}
\icmlaffiliation{sch}{California Institute of Technology,}
\icmlaffiliation{xxx}{ Johns Hopkins University,}
\icmlaffiliation{comp}{UC Berkeley}

\icmlcorrespondingauthor{}{shangding.gu@berkeley.edu}

\icmlkeywords{Machine Learning, ICML}

\vskip 0.3in
]

\printAffiliationsAndNotice{} 

\begin{abstract}
LLM-based multi-agent systems (MAS) have emerged as a promising approach to tackle complex tasks that are difficult for individual LLMs. A natural strategy is to scale performance by increasing the number of agents; however, we find that such scaling exhibits strong diminishing returns in homogeneous settings, while introducing heterogeneity (e.g., different models, prompts, or tools) continues to yield substantial gains. This raises a fundamental question: \emph{what limits scaling, and why does diversity help?}
We present an information-theoretic framework showing that MAS performance is bounded by the intrinsic task uncertainty, not by agent count. We derive architecture-agnostic bounds demonstrating that improvements depend on how many \emph{effective channels} the system accesses. Homogeneous agents saturate early because their outputs are strongly correlated, whereas heterogeneous agents contribute complementary evidence. We further introduce $K^*$, an \emph{effective channel count} that quantifies the number of effective channels without ground-truth labels. Empirically, we show that heterogeneous configurations consistently outperform homogeneous scaling: 2 diverse agents can match or exceed the performance of 16 homogeneous agents. Our results provide principled guidelines for building efficient and robust MAS through diversity-aware design. Code and Dataset are available at the link: \url{https://github.com/SafeRL-Lab/Agent-Scaling}.
\end{abstract}

\section{Introduction}

Large language models (LLMs) have achieved remarkable performance across diverse tasks, including reasoning, coding, and open-domain question answering~\citep{wei2022chain, achiam2023gpt}. However, individual LLMs still struggle with complex problems that require multi-step reasoning, diverse perspectives, or complementary expertise~\citep{huang2023large}. To address these limitations, LLM-based multi-agent systems (MAS) have emerged as a promising paradigm. By orchestrating multiple LLM agents through communication, coordination, or aggregation mechanisms, MAS can tackle challenges that are difficult for single models~\citep{wu2024autogen, hong2024metagptmetaprogrammingmultiagent, du2023improvingfactualityreasoninglanguage}. Recent studies have demonstrated that multi-agent collaboration can yield substantial improvements over single-agent baselines on tasks ranging from software engineering~\citep{qian2024chatdev} to scientific reasoning~\citep{guo2024large}.

\begin{figure}[t]
    \centering
    \includegraphics[width=0.9\linewidth]{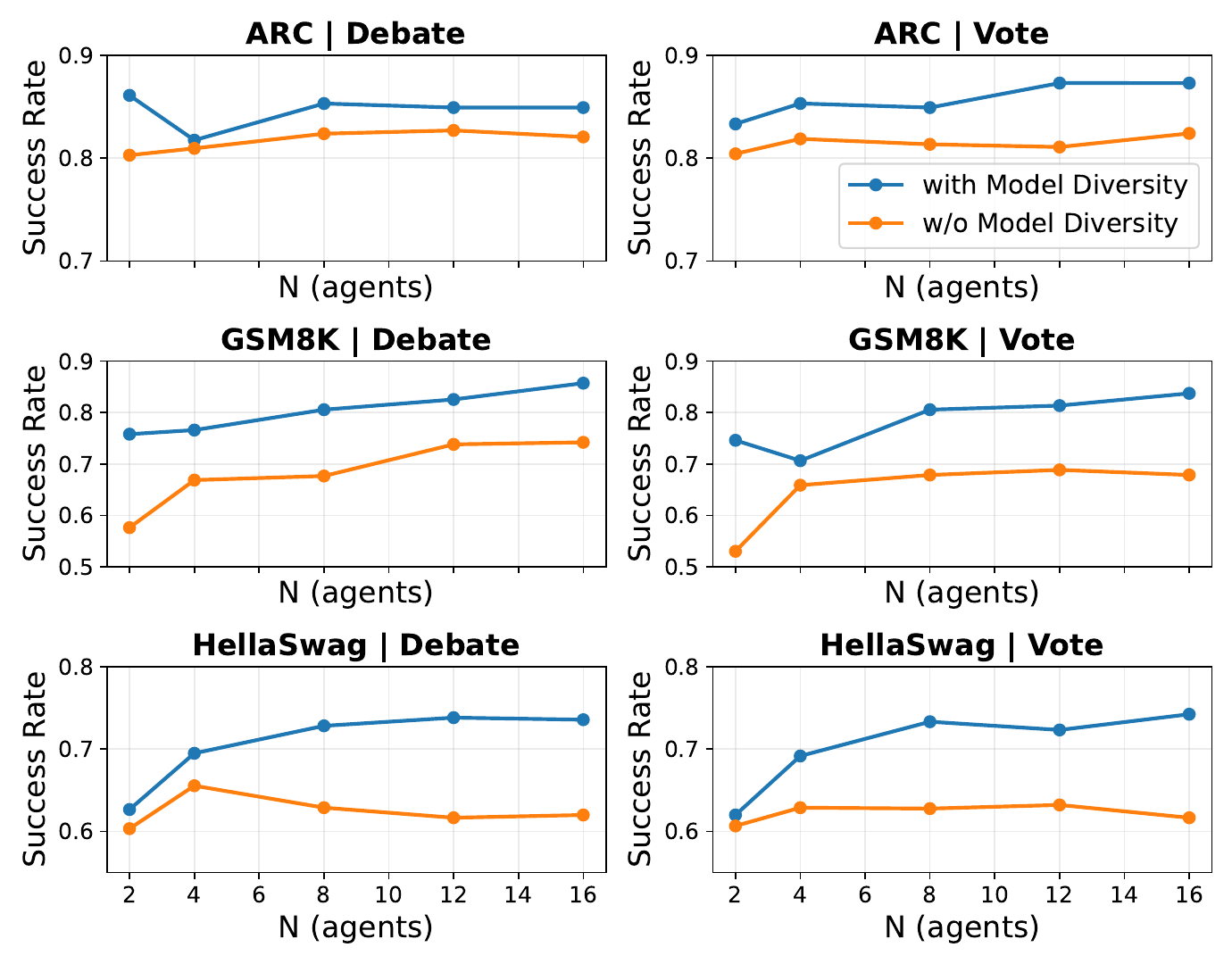}
    \vspace{-0.4cm}
    \caption{
        Effect of model diversity. We compare a mixture of three LLMs (Qwen-2.5-7B, Llama-3.1-8B, Mistral-7B) with the average of independent single-LLM runs.
    }
    \label{fig:mixture_vs_independent}
    \vspace{-0.6cm}
\end{figure}

\begin{figure*}[t]
    \centering
    \includegraphics[width=0.98\linewidth]{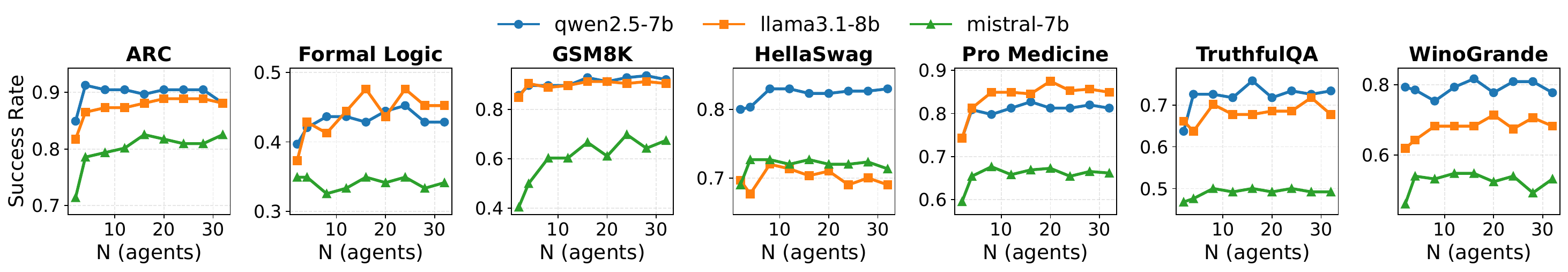}
    \vspace{-0.3cm}
    \caption{
     Scaling behavior of homogeneous multi-agent voting. Success rate versus agent count N on seven tasks for three base models. Performance improves with N but saturates, indicating clear diminishing marginal gains at larger agent counts.
    }    
    \label{fig:homogeneous_scaling}
    \vspace{-0.2cm}
\end{figure*}

Given the effectiveness of multi-agent systems, a natural question arises: \emph{can we improve MAS performance simply by scaling the number of agents?} Intuitively, one might expect ensemble-style gains from aggregating more agent outputs~\citep{li2024agentsneed, wang2023selfconsistencyimproveschainthought}. However, recent work~\citep{kim2025sciencescalingagentsystems} and our experiments reveal a more nuanced picture. As shown in Figure~\ref{fig:homogeneous_scaling}, scaling homogeneous agents (identical models, prompts, and configurations) exhibits strong diminishing returns: accuracy improves at small agent counts, but the marginal gain per additional agent, rapidly collapses toward zero. This suggests that simply adding more homogeneous agents (or allocating more test-time compute) does not reliably introduce new \emph{usable evidence} into the system, but may instead produce increasingly redundant trajectories.

In contrast, our experiments (Figure~\ref{fig:mixture_vs_independent}) show that introducing \emph{diversity} yields sustained performance improvements. Here, diversity broadly refers to heterogeneity in agent configurations, such as backbone models, prompts or personas, and tool access, which empirically leads to more complementary, rather than redundant, information being introduced into the system. As a result, diverse systems can outperform homogeneous ones even with substantially fewer agent calls~\citep{wang2024mixture, zhang2024diversity, qian2025scaling}. Motivated by these observations, we ask: \textbf{what fundamentally limits scaling, and why does diversity help?}

We hypothesize that the primary bottleneck arises from \emph{correlation among agent outputs}. Higher correlation induces greater redundancy, reducing the number of effective channels and leading to performance saturation~\citep{chen2024llm, choi2025debate}. This intuition is illustrated in Figure~\ref{fig:info_coverage}, where heterogeneous agents provide complementary coverage and better information processing diversity compared to homogeneous systems~\citep{yuen2025intrinsic, tang2025importance}. To formalize this intuition, we develop an information-theoretic framework that characterizes MAS performance in terms of \emph{effective channels}, the number of independent, non-redundant reasoning paths present in agent outputs, rather than raw agent count. For example, two agents that reason in nearly identical ways contribute only one effective channel, whereas two agents that follow genuinely different reasoning paths contribute two. Our analysis reveals that performance is bounded by intrinsic task uncertainty, and improvements depend on how many effective channels the system accesses.

Based on this framework, we introduce $K^*$, a label-free metric that quantifies effective channels without requiring ground-truth labels. Empirically, we demonstrate that heterogeneous configurations consistently outperform homogeneous scaling: with only \textbf{2 diverse agents}, we match or exceed the performance of \textbf{16 homogeneous agents}, achieving significant improvement across seven benchmarks.

Although diminishing returns in scaling have been observed empirically~\citep{wang2024a, kim2025sciencescalingagentsystems}, a unified theoretical framework explaining why and when this phenomenon occurs across different MAS workflows, such as voting~\citep{wang2023selfconsistencyimproveschainthought}, debate~\citep{du2023improvingfactualityreasoninglanguage, 10.5555/3692070.3693020}, and centralized orchestration~\citep{hong2024metagptmetaprogrammingmultiagent}, remains lacking. Existing studies offer limited theoretical insight into how evidence accumulation is affected by agent redundancy as the system scales.
To address this gap, we provide a unified information-theoretic explanation for diminishing returns in LLM-based MAS. Our contributions are summarized as follows:
\vspace{-0.3cm}
\begin{itemize}[leftmargin=1.5em] \setlength{\itemsep}{-0.1em}
    \item We derive architecture-independent performance bounds, demonstrating that MAS effectiveness is constrained by the intrinsic task uncertainty $H(Y|X)$, and that improvements arise from increasing the number of effective channels rather than scaling the agent count.
    \item We analyze representative MAS paradigms (vote, debate), showing that homogeneous configurations quickly saturate due to highly correlated evidence, whereas heterogeneity effectively reduces redundancy and expands the system's capacity for effective channels.
    \item We introduce $K^*$, an effective channel count that quantifies the number of non-redundant information sources in agent outputs. We empirically validate that $K^*$ tracks performance and provides principled guidelines for diversity-driven MAS design.
\end{itemize}

\section{Related Works}
\paragraph{Information-Theoretic Analysis of LLM Reasoning.}
Recent work has begun applying information theory to understand LLM behavior. \citet{ton2024understanding} quantify information gain at each chain-of-thought step, showing that effective reasoning requires each step to contribute new information. \citet{gan2025rethinking} analyze cascading failures through information loss accumulation: when $I(t_\ell; r_\ell)$ grows super-linearly, conditional entropy increases rather than decreases. In multi-agent settings, \citet{riedl2025emergent} use Time-Delayed Mutual Information to detect coordination vs.~mere information sharing, and \citet{chang2025multi} track when agent dialogues converge vs.~maintain distributed information. 
However, these works focus on \emph{characterizing} information flow patterns rather than \emph{explaining} why diversity constraints emerge or \emph{deriving} performance bounds. In contrast to prior work that primarily measures or characterizes information flow in LLM reasoning, we use information theory to \emph{explain} diminishing returns via formal limits.

\vspace{-0.3cm}
\paragraph{LLM-based Multi-Agent Systems.}
LLM–based MAS instantiate multiple interacting LLM agents to perform compound inference through communication, coordination, or aggregation mechanisms \citep{xi2023risepotentiallargelanguage,wang2024a,guo2024large}. Existing designs span independent sampling and voting schemes related to self-consistency \cite{wang2023selfconsistencyimproveschainthought}, decentralized debate and role-playing frameworks \cite{du2023improvingfactualityreasoninglanguage,10.5555/3692070.3693020,li-etal-2024-improving-multi,li2024agentsneed,camel}, centralized orchestration frameworks such as AutoGen \cite{wu2024autogen} and MetaGPT \cite{hong2024metagptmetaprogrammingmultiagent}, as well as hybrid, evolving, or self-improving coordination strategies \cite{dang2025multiagent,zhao2025sirius}. \citet{cemri2025why} identify systematic failure modes across multi-agent systems. Taken together, these findings suggest that MAS performance is influenced by multiple design factors, among which we specifically focus on the role of agent diversity.

\vspace{-0.3cm}
\paragraph{Empirical Studies of MAS Scaling and Diversity.}
It has been shown that naïvely scaling the number of agents yields limited benefits when agent behaviors are homogeneous \cite{wang2024a,chen2024llm}, across majority voting \cite{qian2025scaling}, debate \cite{choi2025debate}, and more general coordination mechanisms \cite{kim2025sciencescalingagentsystems}.

In contrast, a growing body of empirical evidence highlights the central role of diversity in MAS. 
\citet{zhang2024diversity} demonstrates that diversity leads to higher success rates in software engineering agents, while \citet{wang2024mixture} finds that heterogeneous
ensembles outperform homogeneous ones. \citet{wu2025hidden} argue that preserving disagreement is preferable to enforcing early consensus. Related work    
shows that diversity benefits depend on task complexity \citep{tang2025importance} and that persona-based diversification has limitations \citep{samuel2024personagym,   
taillandier2025integrating}. However, these findings are restricted to specific diversity forms and narrow settings. We provide a unified theoretical and empirical      
analysis across multiple diversity types. 

\begin{figure*}[t]
\centering
\includegraphics[width=\linewidth]{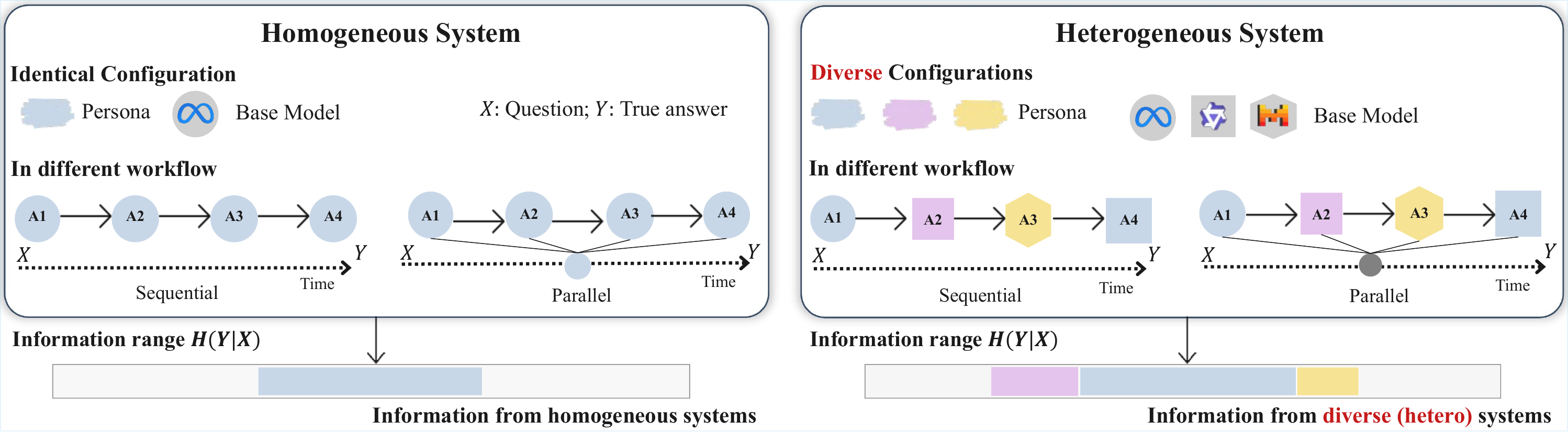}
\vspace{-0.5cm}
\caption{
A comparison between homogeneous and heterogeneous systems. In a homogeneous system, agents with identical configurations result in redundant behavior and limited information coverage. In contrast, heterogeneous agents, through diverse configurations (e.g., varying models or personas), provide complementary coverage and better diversity in the information processed, allowing for more effective problem-solving across different workflows.
}
\label{fig:info_coverage}
\vspace{-0.2cm}
\end{figure*}

\section{Problem Formulation}
\label{sec:preliminaries}

This section formalizes the notion of information flow in LLM-based multi-agent systems and establishes the theoretical foundations for understanding scaling behavior.

We first define the system setup, then introduce the key quantity that governs MAS performance: \emph{usable evidence}. Finally, we derive upper bounds showing that achievable information gain is determined by agent diversity.

\subsection{LLM-based Multi-Agent Systems}
\label{sec:mas-definition}

We begin by formally defining the class of systems we study.

\begin{definition}[LLM-based Multi-Agent System]
\label{def:llm-mas}
An \emph{LLM-based multi-agent system} consists of $N$ agents, each characterized by a \emph{configuration} that specifies its backbone model, system prompt or persona, decoding strategy, and tool access. Given a task input $X$, the system executes a total of $n$ agent calls through a specified workflow (e.g., parallel voting, sequential debate) and aggregates the outputs to produce a final answer.
\end{definition}

\vspace{-0.2cm}
\noindent\textbf{Notation.} We distinguish the \emph{number of agents} $N$ and the \emph{number of agent calls} $n$. In single-round workflows such as majority voting, $n = N$. In multi-round workflows such as debate with $R$ rounds, $n= N \times R$.
This distinction is important because our analysis focuses on how much information is extracted, regardless of which agent produces it. Agent configuration types are formally defined in Section~\ref{sec:config-types}.

\subsection{Usable Evidence and Information Budget}
\label{sec:info-budget}

Consider a task with input $X \in \mathcal{X}$ and ground-truth answer $Y \in \mathcal{Y}$. During inference, the MAS executes $n$ agent calls and produces a dialogue transcript:
\begin{equation}
Z_{1:n} = (Z_1, \ldots, Z_n),
\label{eq:transcript}
\end{equation}
where each output $Z_i$ may depend on the input $X$ and all preceding outputs $Z_{<i} = (Z_1, \ldots, Z_{i-1})$.

The central question is: \emph{how much information about the answer $Y$ can the system extract from its agent calls?} We quantify this through the conditional mutual information:
\vspace{-0.1cm}
\begin{align}\label{eq:imas}
  I_{\mathrm{MAS}}(n)
  ~&:=~
  I(Z_{1:n};Y\mid X)
  \nonumber\\
  ~&=~
  H(Y\mid X)-H(Y\mid X,Z_{1:n}).
\end{align}
This quantity, which we call \emph{usable evidence}, measures the reduction in uncertainty about $Y$ achieved by observing the transcript, beyond what is already contained in the input $X$.

To understand how usable evidence accumulates, let $\Delta_i := I(Z_i;Y\mid X,Z_{<i})$ denote the \emph{incremental contribution} of the $i$-th call, i.e., the new information it provides given all previous outputs. By the chain rule for mutual information:
\vspace{-0.2cm}
\begin{equation}
I_{\mathrm{MAS}}(n)=\sum_{i=1}^n \Delta_i.
\label{eq:delta_def}
\end{equation}
\vspace{-0.5cm}

This decomposition shows that MAS performance depends not on the total number of calls $n$, but on how much \emph{non-redundant} evidence each call contributes. If agents produce highly correlated outputs, the incremental contributions $\Delta_i$ diminish rapidly, leading to saturation. As illustrated in Figure~\ref{fig:info_coverage}, heterogeneous agents provide complementary coverage and better diversity in information processing compared to homogeneous configurations.

The following theorem gives an upper bound on the achievable information:

\begin{theorem}[Finite Information Budget]
\label{thm:finite_budget}
For any transcript $Z_{1:n}$,
\begin{equation}
I_{\mathrm{MAS}}(n)\le H(Y\mid X).
\label{eq:finite_budget}
\end{equation}
\end{theorem}
\vspace{-0.2cm}
\noindent This bound states that no MAS can extract more information about $Y$ than the intrinsic task uncertainty $H(Y|X)$. The practical implication is that scaling benefits plateau once this ceiling is approached, and homogeneous systems may reach saturation much earlier than heterogeneous systems due to redundant evidence.

\subsection{Agent Configuration Types}
\label{sec:config-types}

Agent diversity is operationalized through variations in backbone model, system prompt/persona, decoding strategy, and tool access.
We index these choices by \emph{configuration types}.

\begin{definition}[Agent Configuration Type]
\label{def:config-type}
Each call $i\in\{1,\ldots,n\}$ is associated with a type $b(i)\in\mathcal{B}$.
For each $b\in\mathcal{B}$, define the number of calls
\begin{equation}
m_b \coloneqq \bigl|\{i\in\{1,\ldots,n\}: b(i)=b\}\bigr|,
\quad \sum_{b\in\mathcal{B}} m_b = n.
\label{eq:mb_def}
\end{equation}

\end{definition}

\subsection{Type-Dependent Ceilings Across workflows}
\label{sec:type_upper_bounds}

We next state representative upper bounds showing that, across common MAS workflows, achievable information gain is controlled by the multiset of configuration types.
All formal assumptions and proofs are deferred to Appendix~\ref{sec:appendix}.

\vspace{-0.2cm}
\paragraph{Parallel interaction.}
Let $I_b := I(Z^{(b)};Y\mid X)$ denote the single-call information of type $b$ (see Appendix~\ref{app:parallel} for the full derivation).
Under a standard conditional-independence model for parallel sampling (Assumption~\ref{assump:parallel-ci}),
\begin{equation}
I_{\mathrm{MAS}}^{\mathrm{parallel}}(n)
\;\le\;
H(Y\mid X) \;\wedge\; \sum_{b\in\mathcal{B}} m_b\, I_b.
\label{eq:parallel_upper_main}
\end{equation}

\vspace{-0.2cm}

\paragraph{Sequential interaction.}
Define the maximal per-step contribution of type $b$ by
\vspace{-0.2cm}
\begin{equation}
I_b^{\max}
~:=~
\sup_{z_{<i}}
I(Z_i;Y\mid X,Z_{<i}=z_{<i},\, b(i)=b).
\label{eq:imax_main}
\end{equation}

\vspace{-0.2cm}
Any sequential MAS satisfies
\vspace{-0.2cm}
\begin{equation}
I_{\mathrm{MAS}}^{\mathrm{seq}}(n)
\;\le\;
H(Y\mid X) \;\wedge\; \sum_{b\in\mathcal{B}} m_b\, I_b^{\max}.
\label{eq:sequential_upper_main}
\end{equation}

\vspace{-0.4cm}
Debate is a special case of sequential interaction and inherits the same ceiling.

\vspace{-0.2cm}
\paragraph{From ceilings to compute.}
The bounds above depend on \emph{structural properties} of the MAS, which types are instantiated and how they are composed, rather than on the raw call count $n$.
Since these upper bounds do not depend on $n$, the raw call count is not the right quantity for characterizing MAS performance limits. This motivates us to identify a new quantity, the \emph{effective channel count}, that more directly governs how much usable evidence a MAS can extract.

\section{Why Diversity Matters}
\label{sec:K_theory_practice}

Section~\ref{sec:preliminaries} establishes that MAS performance is bounded by intrinsic task uncertainty and that the upper bounds depend on configuration types rather than the raw call count $n$.
This raises a natural question: what quantity \emph{does} govern how much information a MAS actually extracts? In this section, we introduce the \emph{effective channel count} $K$ to answer this question. We then show why homogeneous scaling often saturates (because $K$ stops growing), while heterogeneous designs can keep improving by increasing the amount of \emph{complementary} evidence (larger $K$, and/or a higher evidence-coverage rate $\alpha$), which leads to a characteristic fast-then-slow gain curve.

\subsection{Effective Channels: From Compute to Usable Evidence}
\label{sec:effective-channels}

An \emph{effective channel} represents one independent source of task-relevant information in the MAS transcript. Intuitively, if two agents produce nearly identical reasoning, they contribute only one effective channel despite consuming two agent calls; if they reason along genuinely different paths, they contribute two. The \emph{effective channel count} $K$ thus captures how many non-redundant information sources the system has, as opposed to the raw number of calls $n$.
To formalize this, we introduce two interrelated concepts: the \emph{complementarity rate} $\alpha$ and the \emph{effective channel count} $K$.

\begin{definition}[Complementarity Rate]
\label{def:alpha}
The \emph{complementarity rate} $\alpha \in (0,1)$ quantifies the probability that a new effective channel uncovers previously missing task-relevant evidence. Formally, $\alpha$ governs the rate at which additional channels reduce residual uncertainty about $Y$.
\end{definition}

Intuitively, $\alpha$ reflects how ``complementary'' the information from different channels is. A high $\alpha$ indicates that each new channel is likely to provide fresh evidence, while a low $\alpha$ suggests substantial overlap with existing information.

\begin{definition}[Effective Channel Representation]
\label{def:eff-ch}
An \emph{effective channel representation} of the transcript $Z_{1:n}$ is a collection of $K$ channels:
\begin{equation}
\tilde Z_{1:K}=(\tilde Z^{(1)},\ldots,\tilde Z^{(K)})
\quad\text{s.t.}\quad
\tilde Z_{1:K}=\phi(Z_{1:n}),
\label{eq:phi_def}
\end{equation}
for some (possibly lossy) aggregation map $\phi$, where $K$ is the \emph{effective channel count}, representing the number of non-redundant information sources in the agent outputs.
\end{definition}

$K$ and $\alpha$ are coupled: increasing diversity (larger $K$) is beneficial only if the new channels provide complementary evidence (captured by $\alpha$). The product $\alpha K$ thus serves as the fundamental quantity governing information recovery, as formalized in Theorem~\ref{thm:lb_saturated}.

Since $\tilde Z_{1:K}$ is a function of $Z_{1:n}$, the data processing inequality implies:
\begin{equation}
I(Z_{1:n};Y\mid X)\;\ge\; I(\tilde Z_{1:K};Y\mid X).
\label{eq:dpi}
\end{equation}

\paragraph{Connecting $K$ and $\alpha$ to recoverable information.}
To formalize the relationship between effective channels and information recovery, we introduce in Appendix~\ref{sec:appendix} a minimal evidence-coverage model (Assumptions~\ref{assump:evidence-bits} and~\ref{assump:coverage}). Under this model, the information recovered from $K$ effective channels with complementarity rate $\alpha$ approaches the intrinsic task uncertainty at a geometric rate:
\begin{theorem}[Geometric Contraction with Effective Channels]
\label{thm:lb_saturated}
Under Assumptions~\ref{assump:evidence-bits} and~\ref{assump:coverage}, the residual uncertainty after observing $K$ effective channels satisfies
\begin{equation}
\boxed{
\begin{aligned}
&H(Y\mid X) - \EE\!\left[I(\tilde Z_{1:K};Y\mid X)\right] \\
&\quad\le\; (1-\alpha)^K \, H(Y\mid X)
\;\le\; e^{-\alpha K} \, H(Y\mid X).
\end{aligned}
}
\label{eq:lower_bound_main}
\end{equation}
Equivalently, the \emph{normalized residual} satisfies
$\EE[H(Y\mid X,\tilde Z_{1:K})]/H(Y\mid X)\le (1-\alpha)^K \le e^{-\alpha K}$.
\end{theorem}

\subsection{$K$ as the State Variable of MAS Scaling}
\label{sec:K_state_variable}

The central question in MAS scaling is not whether $n$ increases, but whether $n$ induces growth in the \emph{effective channel count} $K(n)$.
This follows from Section~\ref{sec:preliminaries}: ceilings are fixed by intrinsic uncertainty $H(Y\mid X)$ and structural design (Section~\ref{sec:type_upper_bounds}), while achievability improves with the number of non-redundant channels (Section~\ref{sec:effective-channels}).
\vspace{-0.2cm}

\paragraph{A direct heterog--homog advantage bound.}
Consider two designs under the same compute budget $n$.
Let $(K_{\mathrm{homog}},\alpha_{\mathrm{homog}})$ and $(K_{\mathrm{heterog}},\alpha_{\mathrm{heterog}})$ denote their effective channel counts and coverage rates in the evidence-coverage model.
By Theorem~\ref{thm:lb_saturated}, each design admits a lower bound on recoverable information:
\begin{corollary}[Heterogeneity Advantage]
\label{cor:heterog_homog_gap_B}
Under Assumptions~\ref{assump:evidence-bits} and~\ref{assump:coverage}, the lower bounds on recoverable information for the two designs are:
\begin{align}
\EE\big[I_{\mathrm{heterog}}\big] &\;\ge\; H(Y\mid X)\big(1-e^{-\alpha_{\mathrm{heterog}}K_{\mathrm{heterog}}}\big), \label{eq:lb_heterog}\\
\EE\big[I_{\mathrm{homog}}\big] &\;\ge\; H(Y\mid X)\big(1-e^{-\alpha_{\mathrm{homog}}K_{\mathrm{homog}}}\big). \label{eq:lb_homog}
\end{align}
When $\alpha_{\mathrm{heterog}}K_{\mathrm{heterog}}>\alpha_{\mathrm{homog}}K_{\mathrm{homog}}$, the heterogeneous design enjoys a strictly higher information-recovery guarantee: its lower bound on recoverable information, $H(Y\mid X)(1-e^{-\alpha_{\mathrm{heterog}}K_{\mathrm{heterog}}})$, exceeds the corresponding homogeneous guarantee $H(Y\mid X)(1-e^{-\alpha_{\mathrm{homog}}K_{\mathrm{homog}}})$.
\end{corollary}
\noindent This is consistent with our empirical findings: as shown in Figure~\ref{fig:mixture_vs_independent} and Table~\ref{tab:hetero_with_baseline}, heterogeneous configurations consistently recover more task-relevant information than homogeneous ones under matched compute.
The corollary formalizes the intuition that heterogeneity helps by increasing $\alpha K$ through more non-redundant channels or higher complementarity.

\begin{table*}[t]
\centering
\caption{Effect of persona diversity. $\Delta$ denotes improvement from heterogeneity. All agents share the same base model pool (Qwen-2.5-7B, Llama-3.1-8B, and Mistral-7B); only persona assignments differ between Homog and Heterog.} 
\vspace{-0.2cm}
\label{tab:hetero_with_baseline}
\scriptsize
\renewcommand{\arraystretch}{1.05}
\setlength{\tabcolsep}{4pt} %
\begin{tabular}{c c c c c >{\columncolor{blue!10}}c c c >{\columncolor{blue!10}}c | c c c c c >{\columncolor{blue!10}}c c c >{\columncolor{blue!10}}c}
\toprule
\multirow{2}{*}{\textbf{Dataset}} & \multirow{2}{*}{\raisebox{0pt}[5ex][0pt]{\shortstack{\textbf{Single}\\\textbf{Agent}}}}
 & \multirow{2}{*}{\textbf{$N$}} & \multicolumn{3}{c}{\textbf{Vote}} & \multicolumn{3}{c|}{\textbf{Debate}} & \multirow{2}{*}{\textbf{Dataset}} & \multirow{2}{*}{\raisebox{0pt}[5ex][0pt]{\shortstack{\textbf{Single}\\\textbf{Agent}}}} & \multirow{2}{*}{\textbf{$N$}} & \multicolumn{3}{c}{\textbf{Vote}} & \multicolumn{3}{c}{\textbf{Debate}} \\
\cmidrule(lr){4-6} \cmidrule(lr){7-9} \cmidrule(lr){13-15} \cmidrule(lr){16-18}
& & & \text{Homog} & \text{Heterog} & \text{$\Delta$} & \text{Homog} & \text{Heterog} & \text{$\Delta$} & & & & \text{Homog} & \text{Heterog} & \text{$\Delta$} & \text{Homog} & \text{Heterog} & \text{$\Delta$} \\
\midrule
\multirow{5}{*}{\textbf{GSM8K}} 
& \multirow{5}{*}{50.8}
& 2  & 86.5 & 87.3 & +0.8 & 76.2 & 75.4 & -0.8 
& \multirow{5}{*}{\textbf{ARC}}
& \multirow{5}{*}{77.8}
& 2  & 78.6 & 81.8 & +3.2 & 84.9 & 87.3 & +2.4 \\
& & 4  & 84.9 & 88.1 & +3.2 & 73.8 & 79.4 & +5.6 & & & 4  & 79.4 & 85.7 & +6.3 & 79.4 & 84.1 & +4.8 \\
& & 8  & 90.5 & 93.7 & +3.2 & 75.4 & 85.7 & +10.3 & & & 8  & 84.1 & 86.5 & +2.4 & 84.9 & 85.7 & +0.8 \\
& & 12 & 86.5 & 90.5 & +4.0 & 77.8 & 87.3 & +9.5 & & & 12 & 85.7 & 89.7 & +4.0 & 82.5 & 87.3 & +4.8 \\
& & 16 & 89.7 & 92.1 & +2.4 & 83.3 & 88.1 & +4.8 & & & 16 & 84.9 & 88.9 & +4.0 & 84.9 & 84.9 & 0.0 \\
\midrule
\multirow{5}{*}{\shortstack{\textbf{Formal}\\\textbf{Logic}}}
& \multirow{5}{*}{32.0}
& 2  & 45.2 & 48.4 & +3.2 & 34.1 & 38.9 & +4.8 
&\multirow{5}{*}{\shortstack{\textbf{Truthful}\\\textbf{QA}}}
& \multirow{5}{*}{71.8}
& 2  & 74.2 & 77.4 & +3.2 & 71.0 & 77.4 & +6.4 \\
& & 4  & 47.6 & 52.4 & +4.8 & 42.9 & 53.2 & +10.3 & & & 4  & 75.0 & 75.8 & +0.8 & 71.8 & 79.8 & +8.0 \\
& & 8  & 47.6 & 55.6 & +7.9 & 49.2 & 53.2 & +4.0 & & & 8  & 76.6 & 79.0 & +2.4 & 76.6 & 78.2 & +1.6 \\
& & 12 & 48.4 & 57.9 & +9.5 & 48.4 & 54.8 & +6.4 & & & 12 & 75.0 & 79.0 & +4.0 & 73.4 & 79.8 & +6.4 \\
& & 16 & 50.0 & 54.0 & +4.0 & 43.6 & 51.6 & +8.0 & & & 16 & 78.2 & 81.5 & +3.3 & 75.0 & 84.7 & +9.7 \\
\midrule
\multirow{5}{*}{\textbf{HellaSwag}} 
& \multirow{5}{*}{66.1}
& 2  & 62.3 & 73.7 & +11.4 & 50.3 & 75.0 & +24.7 
&\multirow{5}{*}{\shortstack{\textbf{Wino}\\\textbf{grande}}}
& \multirow{5}{*}{57.1}
& 2  & 51.6 & 60.3 & +8.7 & 58.7 & 50.0 & -8.7 \\
& & 4  & 68.7 & 75.3 & +6.6 & 66.0 & 73.0 & +7.0 & & & 4  & 54.0 & 69.1 & +15.1 & 53.2 & 62.7 & +9.5 \\
& & 8  & 70.0 & 79.0 & +9.0 & 69.7 & 76.0 & +6.3 & & & 8  & 57.9 & 69.1 & +11.2 & 61.9 & 69.1 & +7.2 \\
& & 12 & 72.3 & 79.0 & +6.7 & 69.3 & 78.3 & +9.0 & & & 12 & 58.7 & 70.6 & +11.9 & 62.7 & 70.6 & +7.9 \\
& & 16 & 72.0 & 79.9 & +7.9 & 70.3 & 76.4 & +6.1 & & & 16 & 60.3 & 69.8 & +9.5 & 57.9 & 64.3 & +6.4 \\
\midrule
\multirow{5}{*}{\shortstack{\textbf{Pro}\\\textbf{Medicine}}}
& \multirow{5}{*}{68.6}
& 2  & 78.3 & 78.7 & +0.4 & 76.8 & 71.3 & -5.5 
& \multirow{5}{*}{\textbf{Average}}
& \multirow{5}{*}{60.6}
& 2  & 68.1 & 72.5 & \textbf{+4.4} & 64.6 & 67.9 & \textbf{+3.3} \\
& & 4  & 80.5 & 81.6 & +1.1 & 76.8 & 76.5 & -0.3 & & & 4  & 69.9 & 76.1 & \textbf{+6.2} & 66.3 & 72.7 & \textbf{+6.4} \\
& & 8  & 81.3 & 83.5 & +2.2 & 81.6 & 82.7 & +1.1 & & & 8  & 72.6 & 79.6 & \textbf{+7.0} & 71.3 & 75.8 & \textbf{+4.5} \\
& & 12 & 80.2 & 82.7 & +2.5 & 81.3 & 83.8 & +2.5 & & & 12 & 72.4 & 81.0 & \textbf{+8.6} & 70.8 & 77.4 & \textbf{+6.6} \\
& & 16 & 80.5 & 81.8 & +1.3 & 80.5 & 83.3 & +2.8 & & & 16 & 73.6 & 81.1 & \textbf{+7.5} & 70.8 & 76.2 & \textbf{+5.4} \\
\bottomrule
\end{tabular}
\vspace{-0.2cm}
\end{table*}

\vspace{-0.2cm}
\paragraph{Fast-then-slow scaling: the $1-e^{-\alpha K}$ shape.}
Corollary~\ref{cor:lb} implies that recoverable information grows at least as
\begin{equation}
\EE\big[I(\tilde Z_{1:K};Y\mid X)\big] \;\ge\; H(Y\mid X)\big(1-e^{-\alpha K}\big).
\label{eq:info_curve}
\end{equation}
The shape of~\eqref{eq:info_curve} directly predicts diminishing returns: the marginal gain from one additional effective channel satisfies
\begin{equation}
\big(1-e^{-\alpha (K+1)}\big)-\big(1-e^{-\alpha K}\big)
~=~ (1-e^{-\alpha})\,e^{-\alpha K},
\label{eq:marginal_decay}
\end{equation}
which is largest at small $K$ and decays exponentially thereafter.
This yields a clean explanation for the empirically observed \emph{fast-then-slow} improvement pattern as the number of agents $n$ increases:
early gains occur when $K(n)$ is still growing, while later gains diminish once $K(n)$ saturates.

\subsection{Measuring Effective Channels Without Labels: $K^*$}
\label{sec:Kstar}

The effective channel count $K$ cannot be computed directly at inference time because it depends on the unknown ground-truth $Y$. We therefore introduce $K^*$, a \emph{label-free proxy} that estimates the number of effective channels from agent outputs in embedding space: $K^*$ is large when outputs are diverse and approaches 1 when outputs are similar.

\paragraph{Definition.}
Let $\mathrm{Emb}(\cdot)$ be an embedding model.
Given outputs $\{Z_i\}_{i=1}^n$, define normalized embeddings
\begin{equation}
\hat{\mathbf{z}}_i \;:=\; \frac{\mathrm{Emb}(Z_i)}{\|\mathrm{Emb}(Z_i)\|_2}\in\mathbb{R}^d,
\end{equation}
and the cosine-similarity Gram matrix $G\in\mathbb{R}^{n\times n}$:
\begin{equation}
G_{ij} \;:=\; \langle \hat{\mathbf{z}}_i,\hat{\mathbf{z}}_j\rangle.
\end{equation}
Trace-normalize to obtain $\rho:=G/\mathrm{Tr}(G)$ with $\mathrm{Tr}(\rho)=1$, and let $\{\lambda_j\}_{j=1}^n$ be the eigenvalues of $\rho$.
We define the entropy effective rank
\begin{equation}
K^* \coloneqq 2^{H(\rho)},
\quad \text{where} \quad
H(\rho) = -\sum_{j=1}^n \lambda_j \log_2 \lambda_j.
\label{eq:Kstar_def}
\end{equation}

\vspace{-0.3cm}

\paragraph{Interpretation.}
$K^*$ counts how many ``independent directions'' the agent outputs span in embedding space.
When all agents produce nearly identical outputs (e.g., paraphrases of the same reasoning), their embeddings are collinear and $K^* \approx 1$: the system effectively has a single information channel.
When agents produce genuinely different outputs whose embeddings point in different directions with roughly equal magnitude, $K^*$ grows toward $n$: each agent contributes a distinct channel.
For example, if four agents all solve a math problem using the same algebraic approach with minor wording differences, their outputs cluster in one direction and $K^*\approx 1$. If instead the agents employ genuinely different strategies (e.g., algebraic manipulation, geometric reasoning, and numerical estimation), $K^*$ will be notably larger than~$1$, reflecting that the system draws on multiple independent lines of evidence.
Formally, $1\le K^*\le n$. $K^*$ reaches its maximum $n$ when the normalized Gram matrix $\rho$ has a uniform spectrum (all eigenvalues equal to $1/n$), which occurs when outputs are orthogonal and carry equal energy.
Proofs of these properties are given in Appendix~\ref{sec:appendix}.

\section{Experiments}
\label{sec:experiment}
This section validates 3 core claims:  
(i) scaling \emph{homogeneous} MAS exhibits diminishing returns,  
(ii) \emph{heterogeneity} consistently outperforms pure scaling under matched compute,  
and (iii) performance gains are governed by the \emph{number of effective channels} rather than the raw agent count.

\begin{figure*}[t]
\centering

\begin{subfigure}[b]{0.48\textwidth}
    \centering
    \includegraphics[width=\textwidth]{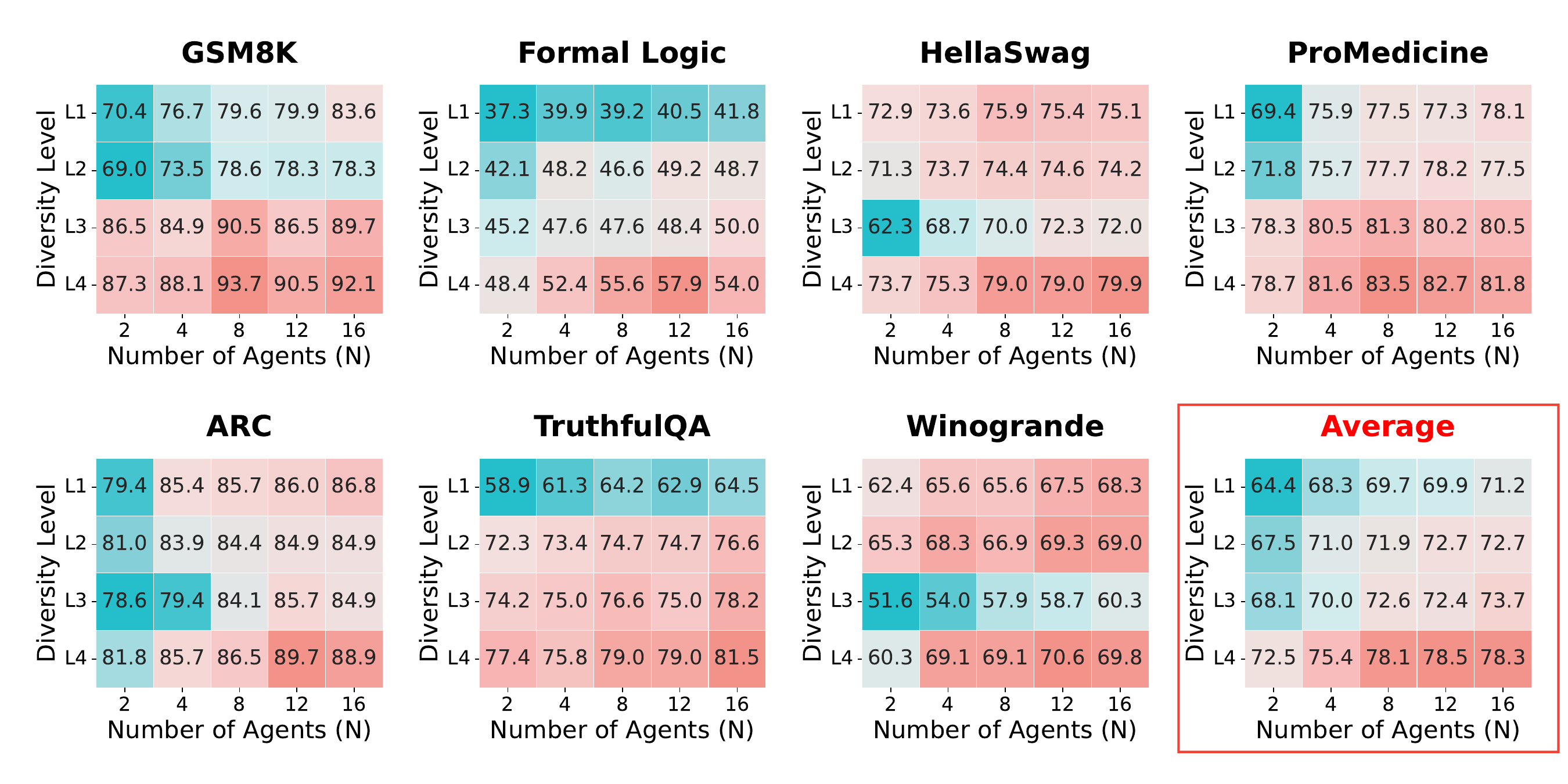}
    \vspace{-0.5cm}
    \caption{Vote}
    \label{Fig3-1}
\end{subfigure}
\quad
\begin{subfigure}[b]{0.48\textwidth}
    \centering
    \includegraphics[width=\textwidth]{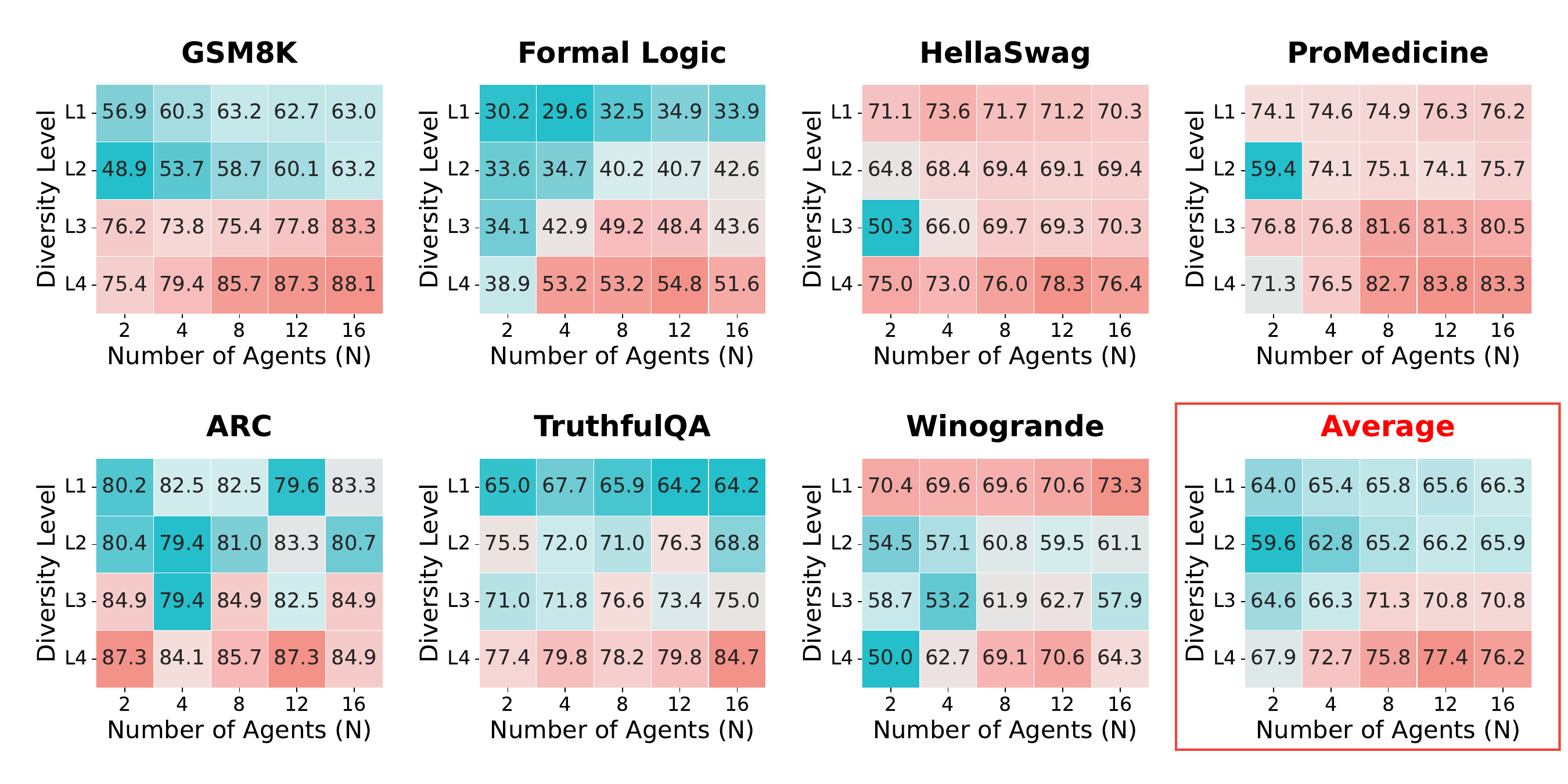}
    \vspace{-0.5cm}
    \caption{Debate}
    \label{Fig3-2}
\end{subfigure}
\vspace{-0.3cm}
\caption{
Diversity analysis of the \textbf{Vote} and \textbf{Debate} mechanisms across all datasets.
Each subfigure corresponds to one dataset and visualizes absolute performance as a $4 \times 5$ heatmap, where rows represent progressively enriched diversity layers (L1--L4) and columns denote the number of agents $N$. Colors indicate success rate values: \textcolor{DeepCyan}{cyan (lowest)} to \textcolor{red}{red (highest)}.
}
\label{fig:Heatmap}
\vspace{-0.3cm}
\end{figure*}

\subsection{Experimental Setup}

\paragraph{Tasks.}
We consider a diverse set of reasoning and knowledge benchmarks, including GSM8K~\citep{cobbe2021gsm8k}, ARC~\citep{allenai:arc}, Formal Logic~\citep{hendrycks2021ethics,hendryckstest2021}, TruthfulQA~\citep{lin2022truthfulqameasuringmodelsmimic}, HellaSwag~\citep{zellers2019hellaswag}, WinoGrande~\citep{ai2:winogrande}, and Pro Medicine~\citep{hendrycks2021ethics,hendryckstest2021}.
These tasks span arithmetic reasoning, formal deduction, commonsense reasoning, and domain knowledge, covering both deterministic and ambiguous settings.
\vspace{-0.2cm}

\paragraph{Models.}
Agents are instantiated using three open-source LLMs: Qwen-2.5-7B~\citep{qwen2.5}, Llama-3.1-8B~\citep{grattafiori2024llama3herdmodels}, and Mistral-7B~\citep{jiang2023mistral7b}.
In the \emph{single-model} setting, all agents within a MAS share the same base model; in the \emph{MIX} setting, agents within a single MAS can use different base models, enabling model-level heterogeneity.
\vspace{-0.2cm}

\paragraph{MAS Workflows.}
We consider two representative collaboration mechanisms~\citep{choi2025debate}: 
\textbf{Vote}, where agents independently generate answers and a majority decision is taken after 1 round, and 
\textbf{Debate}, where agents interact sequentially for 4 rounds before producing a final answer. 
For each mechanism, we vary the number of agents \(N \in \{2,4,8,12,16\}\). 
Compute budgets are matched by fixing the total number of agent calls.
\vspace{-0.2cm}

\paragraph{Diversity Configurations.}
We organize agent heterogeneity into four progressively enriched layers to isolate the contribution of each diversity source:
\vspace{-0.4cm}
\begin{itemize}[leftmargin=2em] \setlength{\itemsep}{-0.2em}
\item \textbf{L1: No Diversity.} All agents share the same base model and the same default system prompt (no persona). This serves as the homogeneous baseline. Results are averaged over the three single-model runs.
\item \textbf{L2: Persona Diversity Only.} All agents share the same base model, but each agent receives a distinct persona prompt (e.g., ``You are an expert mathematician'' vs.\ ``You are a careful logician''). Results are averaged over the three single-model runs.
\item \textbf{L3: Model Diversity Only.} Agents are drawn from different base models (Qwen, Llama, Mistral) but all use the same default system prompt.
\item \textbf{L4: Full Diversity.} Agents differ in both base model and persona prompt, combining model-level and prompt-level heterogeneity.
\end{itemize}
\vspace{-0.3cm}
This controlled design allows us to isolate and compare the contributions of model diversity and persona diversity.

\begin{table}[t]
\centering
\caption{Efficiency gains from diversity. Number of agents needed to match L1 (N=16) baseline. Higher diversity achieves equivalent performance with fewer agents.} 
\vspace{-0.2cm}
\label{tab:efficiency}
\scriptsize
\setlength{\tabcolsep}{5pt}
\begin{tabular}{c|c|c|c|c}
\toprule
\textbf{Method} & \textbf{Config} & \makecell{\textbf{Agents to Match}\\\textbf{L1 (N=16)}} & \makecell{\textbf{Accuracy at}\\\textbf{that N}} & \makecell{\textbf{Peak Accuracy}\\\textbf{(any N)}} \\
\midrule
\multirow{4}{*}{\textbf{Vote}}
& L1 & 16 (baseline) & 65.34 & 65.49 \\
& L2 & 8  & 65.44 & 66.01 \\
& L3 & 4  & 67.29 & 71.54 \\
\rowcolor{green!10}
& L4 & \textbf{2} & 67.71 & \textbf{76.86} \\
\midrule
\multirow{4}{*}{\textbf{Debate}}
& L1 & 16 (baseline) & 65.48 & 65.48 \\
& L2 & 12 & 66.08 & 66.08 \\
& L3 & 4  & 66.26 & 71.33 \\
\rowcolor{green!10}
& L4 & \textbf{2} & 67.90 & \textbf{77.43} \\
\bottomrule
\end{tabular}
\vspace{-0.5cm}
\end{table}

\subsection{Finding 1: Scaling Homogeneous MAS Exhibits Diminishing Returns}
We first examine whether increasing the number of agents improves performance in homogeneous settings.
Figure~\ref{fig:homogeneous_scaling} shows success rates and marginal gains for both voting- and debate-based MAS across multiple tasks and base models.

Across all settings, we observe a consistent pattern:
accuracy improves only at small agent counts, after which marginal gains $\Delta\text{Success}/\Delta N$ rapidly collapse toward zero.
In several cases, performance even degrades as $N$ increases.

As predicted by our theoretical framework (Theorem~\ref{thm:lb_saturated}), this saturation occurs because homogeneous agents produce highly correlated outputs, so additional calls fail to increase the effective channel count $K$. In other words, allocating more test-time computation via homogeneous scaling does not reliably inject new usable evidence into the system.

\subsection{Finding 2: Diversity Consistently Beats Scale}

We compare homogeneous scaling with heterogeneous designs under matched compute in Table~\ref{tab:hetero_with_baseline}, which reports the performance of Vote and Debate mechanisms across all tasks and agent counts. In nearly all cases, heterogeneous configurations significantly outperform homogeneous ones, with gains increasing as $N$ grows. Figure~\ref{fig:Heatmap} provides a detailed view of this effect. Enriching diversity from L1 to L4 yields consistent performance improvements for both Vote and Debate. Notably, model diversity (L3) and persona diversity (L2) each deliver non-trivial gains, while their combination (L4) consistently performs best.

Table~\ref{tab:efficiency} shows the minimum number of heterogeneous agents required to outperform homogeneous configurations. For both Vote and Debate, L4 (full diversity) with just \textbf{2 agents} surpasses the performance of L1 (no diversity) with \textbf{16 agents}. This represents an $8\times$ reduction in agent count for equivalent or better accuracy. This result directly reflects the theory: by Corollary~\ref{cor:heterog_homog_gap_B}, the heterogeneous design achieves a higher $\alpha K$ product, so fewer agents suffice to reach the same information-recovery level.

We also compare heterogeneous model mixtures against independent single-model runs. Figure~\ref{fig:mixture_vs_independent} demonstrates that a mixture of three LLMs outperforms the average performance of the individual models, confirming that the improvements stem from complementary effective channels rather than simple averaging.

\begin{figure}[t]
    \centering
    \includegraphics[width=0.97\linewidth]{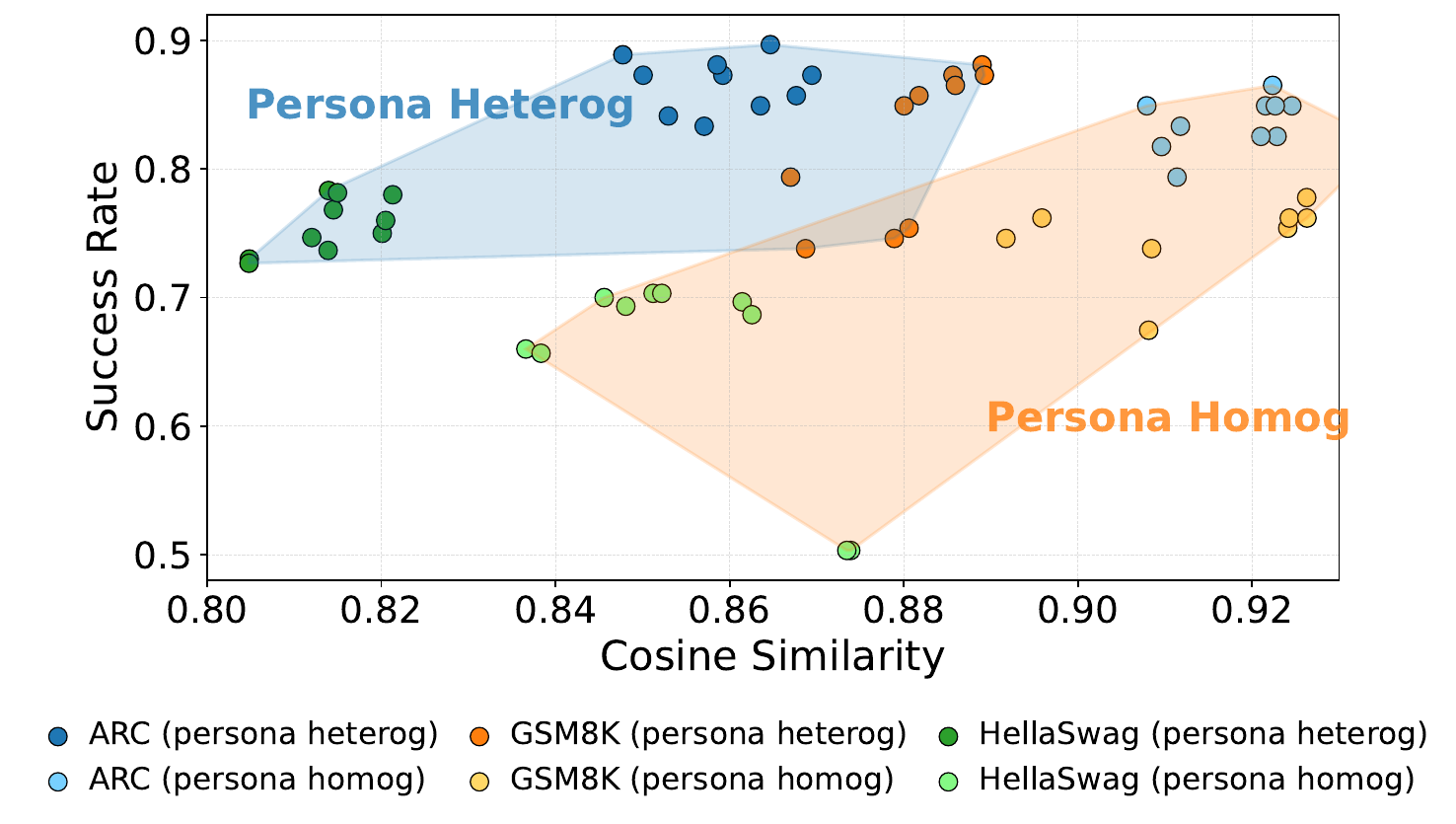}
    \vspace{-0.3cm}
    \caption{Correlation between cosine similarity and success rate. Homogeneous settings show higher similarity but lower performance; heterogeneous personas preserve diversity and improve accuracy.}
    \label{fig:Correlation}
    \vspace{-0.3cm}
\end{figure}

\begin{figure}[t]
    \centering
    \includegraphics[width=0.9\linewidth]{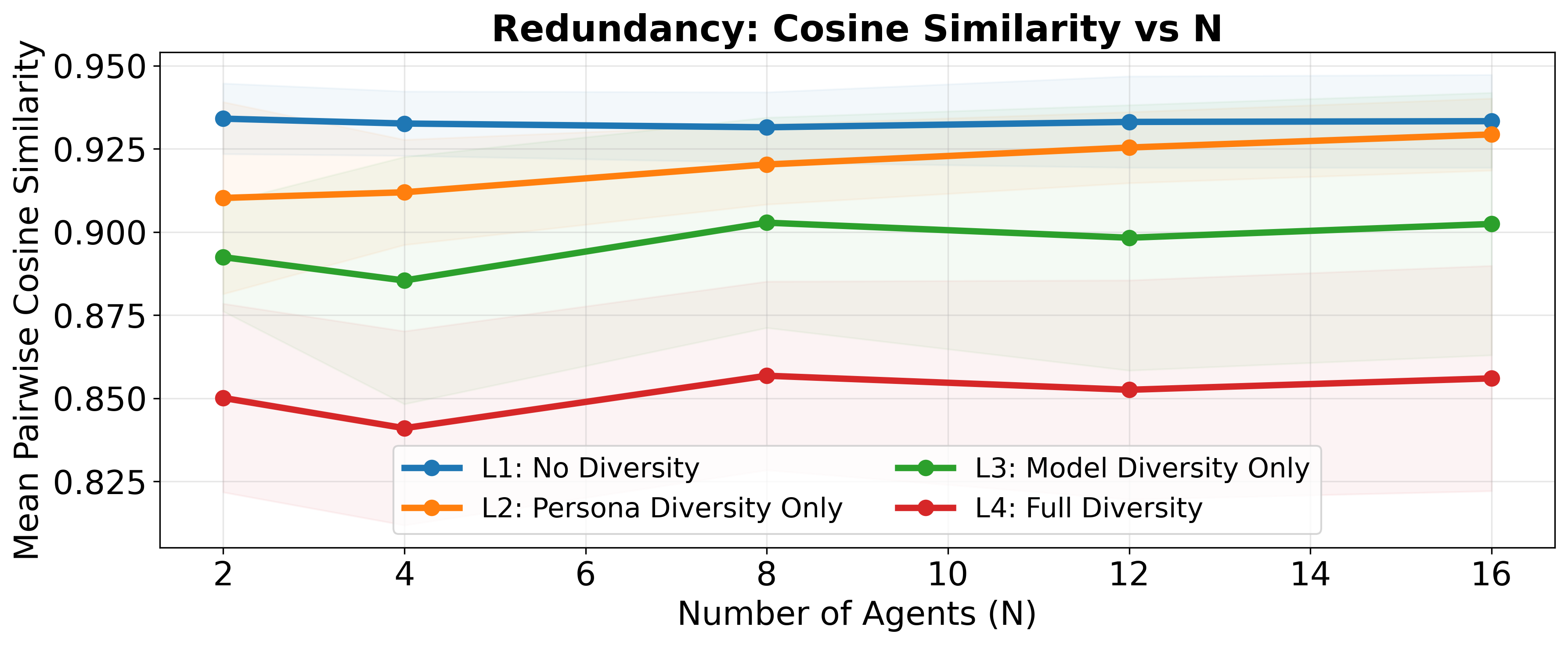}
    \vspace{-0.3cm}
    \caption{Mean pairwise cosine similarity vs. agent count. Higher diversity (L1$\to$L4) consistently reduce redundancy.}
    \label{fig:response_redundancy}
    \vspace{-0.3cm}
\end{figure}

\subsection{Finding 3: Performance Gains Are Governed by the Number of Effective Channels}

Our theory predicts that homogeneous agents produce highly correlated outputs, contributing few effective channels and leading to saturation. We now verify this empirically, proceeding from a simple redundancy proxy (pairwise cosine similarity) to the effective channel measure ($K^*$).

\subsubsection{High output similarity hinders performance}
\label{sec:redundancy_saturation}
A key reason homogeneous scaling saturates is that additional agent calls increasingly produce \emph{correlated} outputs, yielding limited \emph{new} evidence.
To quantify this redundancy, we embed each agent output (the full reasoning trace) using NV-Embed-v2~\citep{lee2025nvembedimprovedtechniquestraining} and compute the \emph{mean pairwise cosine similarity}: for $n$ agent outputs with normalized embeddings $\hat{\mathbf{z}}_1,\ldots,\hat{\mathbf{z}}_n$, this is $\bar{\rho} = \frac{2}{n(n-1)}\sum_{i<j}\langle \hat{\mathbf{z}}_i, \hat{\mathbf{z}}_j\rangle$.
While $\bar{\rho}$ is not an information-theoretic quantity, it provides a consistent proxy for output overlap: higher $\bar{\rho}$ indicates that agents explore fewer non-redundant directions, which constrains the growth of effective channels.

Figure~\ref{fig:Correlation} shows that homogeneous persona settings produce higher similarity yet do not translate this additional compute into higher success rates, whereas heterogeneous personas maintain lower similarity and achieve stronger performance.
Moreover, Figure~\ref{fig:response_redundancy} reveals a systematic scaling trend: for every diversity layer, redundancy increases with agent count $N$, implying that larger homogeneous ensembles mainly amplify existing trajectories rather than introducing qualitatively new evidence.
Crucially, redundancy decreases monotonically from L1 to L4, consistent with our hypothesis that heterogeneity mitigates output correlation and thus enlarges the number of effective channels.

While these results confirm a qualitative relationship between output diversity and performance, pairwise cosine similarity is a coarse measure. To obtain a more precise and theoretically grounded characterization, we next turn to the effective channel count $K^*$ introduced in Section~\ref{sec:Kstar}.

\begin{table}[t]
\centering
\caption{Relation between K* and Accuracy on ARC.}
\vspace{-0.2cm}
\label{tab:arc-Kstar}
\scriptsize
\renewcommand{\arraystretch}{1.2}
\setlength{\tabcolsep}{4.5pt}
\begin{tabular}{ll|cc|cc|cc}
\toprule
\multirow{2}{*}{Method} & \multirow{2}{*}{Config}
& \multicolumn{2}{c|}{\textbf{Performance}}
& \multicolumn{2}{c|}{\textbf{Channels}}
& \multicolumn{2}{c}{\textbf{Answer-Cond.}} \\
\cmidrule(lr){3-4}\cmidrule(lr){5-6}\cmidrule(lr){7-8}
& & Acc. & $\Delta$Acc
& $K^*$ & $\Delta K^*$
& $K^*_c$ & $K^*_w$ \\
\midrule
\multirow{4}{*}{Debate}
& L1 & 81.6\% & --   & 1.197 & --     & 1.184 & 1.177\\
& L2 & 81.0\% & -0.7 & 1.348 & +0.152 & 1.315 & 1.234 \\
& L3 & 83.3\% & +1.7 & 1.246 & +0.049 & 1.220 & 1.160 \\
& L4 & \textbf{85.9\%} & +4.2 & \textbf{1.517} & \textbf{+0.320} & \textbf{1.472} & 1.288  \\
\midrule
\multirow{4}{*}{Vote}
& L1 & 81.3\% & --   & 1.201 & --     & 1.183 & 1.173\\
& L2 & 81.5\% & +0.2 & 1.349 & +0.149 & 1.318 & 1.222 \\
& L3 & 83.8\% & +2.5 & 1.245 & +0.044 & 1.223 & 1.161 \\
& L4 & \textbf{87.5\%} & \textbf{+6.1} & \textbf{1.521} & \textbf{+0.321} & \textbf{1.484} & 1.297 \\
\bottomrule
\end{tabular}
\vspace{-0.5cm}
\end{table}

\subsubsection{Diverse Channels improves performance}
\label{sec:kstar_analysis}

We compute \(K^*\) by embedding each agent output with NV-Embed-v2~\citep{lee2025nvembedimprovedtechniquestraining}, forming the cosine-similarity matrix \(G\), trace-normalizing it to \(\rho\) with \(\mathrm{Tr}(\rho) = 1\), and defining \(K^*\) as the entropy effective rank of \(\rho\) (Eq.~\ref{eq:Kstar_def}).

\vspace{-0.2cm}
\paragraph{Diversity increases \(K^*\).}
As shown in Table~\ref{tab:arc-Kstar}, $K^*$ consistently increases with diversity level from L1 to L4 under both Vote and Debate mechanisms, validating $K^*$ as a robust indicator of system diversity without ground-truth labels.

\paragraph{Higher \(K^*\) leads to better performance.}
The increase in $K^*$ is accompanied by higher accuracy in most cases (Table~\ref{tab:arc-Kstar}). Figure~\ref{fig:L4_kstar_accuracy_mean_fit} further confirms this positive correlation, depicting a strong linear relationship between $K^*$ and task accuracy across configurations.
Moreover, consistent with Theorem~\ref{thm:lb_saturated}, the marginal improvement in accuracy diminishes as $K^*$ grows, reflecting the geometric decay $(1-\alpha)^K$ predicted by our theory. We observe a minor anomaly in L2 under Debate, where $K^*$ increases but accuracy slightly decreases; we investigate this through the decomposition of $K^*$ below.

\paragraph{Mechanistic Decomposition: \(K^*_c\) vs. \(K^*_w\).} 
To determine if the growth in \(K^*\) represents useful evidence or merely increased noise, we decompose it into \(K^*_c\) (correct reasoning diversity) and \(K^*_w\) (incorrect reasoning diversity). Let \(\hat{y}_i\) represent the final answer of agent \(i\), and \(Y\) be the ground-truth label. We define:
\vspace{-0.2cm}
\[
\mathcal{I}_c = \{i: \hat{y}_i = Y\}, \qquad \mathcal{I}_w = \{i: \hat{y}_i \neq Y\}
\]
Here, \(\mathcal{I}_c\) is the set of correct agents, and \(\mathcal{I}_w\) is the set of incorrect agents. We then compute the effective number of channels for each set:\vspace{-0.2cm}
\[
K^*_c = K^*(Z_c), \qquad K^*_w = K^*(Z_w)
\]
where \(Z_c\) and \(Z_w\) are the sub-matrices of the original data matrix \(Z\), corresponding to correct and incorrect agents.

\begin{figure}[t]
    \centering
    \includegraphics[width=0.97\linewidth]{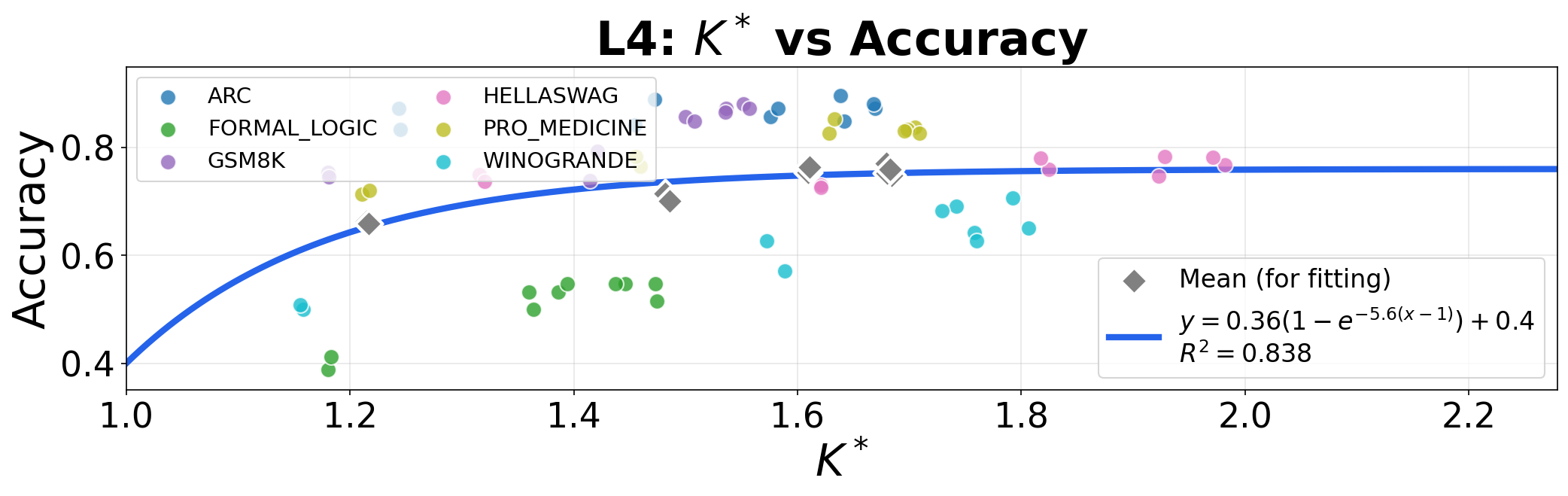}
    \vspace{-0.3cm}
    \caption{Correlation between $K^*$ and accuracy.}
    \label{fig:L4_kstar_accuracy_mean_fit}
\end{figure}

\begin{figure}[t]
    \centering
    \includegraphics[width=\linewidth]{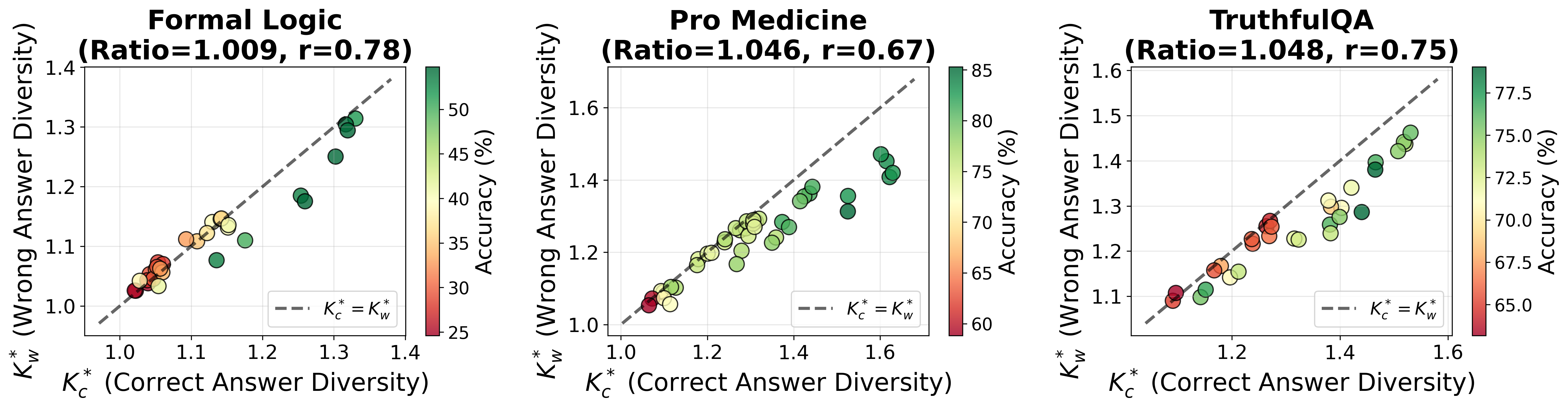}
    \vspace{-0.5cm}
    \caption{Decomposition of \(K^*\) on three tasks. Points below the diagonal correspond to configurations where correct answer diversity dominates.}
    \label{fig:exp1-task-profiles}
    \vspace{-0.2cm}
\end{figure}

\paragraph{The Empirical Boundary.}
Figure~\ref{fig:exp1-task-profiles} suggests an empirical boundary in the $(K^*_c, K^*_w)$ plane:
high-accuracy configurations concentrate in the region where $K^*_c > K^*_w$ (below the diagonal line).
The intuition is as follows: when multiple agents arrive at the correct answer through genuinely \emph{different} reasoning paths ($K^*_c$ is high), the correct answer receives support from independent evidence sources, making it more robust under aggregation. Conversely, when incorrect answers are also diverse ($K^*_w$ is high), the error ``votes'' are spread across many competing alternatives, which can dilute the correct signal.
Thus, $K^*_c > K^*_w$ indicates that correct reasoning benefits from diverse support while incorrect reasoning remains fragmented, and this favorable signal-to-noise ratio is a prerequisite for robust MAS performance.

\subsection{Design Guidelines for LLM-based MAS}
\label{sec:design_guidelines}
Our analysis of effective channels yields several data-driven design guidelines for MAS development:
\vspace{-0.2cm}
\begin{itemize}[leftmargin=2em] 
\setlength{\itemsep}{-0.2em}
    \item \textbf{Match diversity to task type.}  
    $K^*$ predicts accuracy strongly on reasoning tasks but weakly on knowledge-heavy tasks. For tasks requiring complex multi-step reasoning (e.g., \textit{GSM8K}, \textit{ARC}), investing in diversity yields significant performance gains. In contrast, for tasks dominated by factual retrieval (e.g., \textit{Winogrande}), the diversity investment should be more conservative.
    \item \textbf{Ensure correct-path dominance.}
    Systems with high \( K^*_c / K^*_w \) achieve substantially higher accuracy. In practice, this means that when introducing diversity, one should focus on increasing the diversity of \emph{correct} reasoning paths, for example by using personas that encourage different valid problem-solving strategies (e.g., algebraic vs.\ geometric approaches in math tasks), rather than indiscriminately adding diversity that may also amplify incorrect reasoning (e.g., random temperature increases that introduce more errors).
    \item \textbf{Right-size agent count.}  
    Homogeneous systems plateau at \(N \approx 4\), while heterogeneous systems continue to benefit from scaling up to \(N \approx 8\). Beyond this point, adding more agents results in diminishing returns and wasted compute resources. Thus, it is important to find a balance in agent count to avoid inefficiency.
\end{itemize}

\section{Conclusion}

This paper shows that simply increasing agent count in multi-agent systems results in diminishing returns, both for homogeneous and heterogeneous configurations. However, heterogeneity improves performance by introducing more diverse, non-redundant information, delaying saturation. We introduce \(K^*\), a label-free measure of effective channels, which reveals that performance gains are driven by the balance between correct-path diversity and redundancy. These results suggest that the challenge in multi-agent scaling lies in the effective allocation of diverse information channels rather than just raw computational power.

\newpage
\section*{Impact Statement}
This work establishes an information-theoretic framework for understanding scaling behavior in LLM-based multi-agent systems. We discuss the scope, limitations, and implications of our contributions below.
\vspace{-0.2cm}
\paragraph{Theoretical Contributions and Scope.}
Our framework provides architecture-agnostic upper bounds and lower bounds showing that MAS performance is fundamentally limited by diverisity. The geometric contraction result (Theorem~\ref{thm:lb_saturated}) offers a principled explanation for the empirically observed ``fast-then-slow'' pattern. However, our theoretical analysis relies on idealized assumptions: the evidence-bits model (Assumption~\ref{assump:evidence-bits}) assumes perfect sufficiency and conditional independence of latent evidence, and the coverage model (Assumption~\ref{assump:coverage}) assumes uniform and independent coverage probabilities. Real-world agents may exhibit more complex dependency structures.

\vspace{-0.2cm}
\paragraph{Limitations of $K^*$.}
While $K^*$ provides a practical label-free proxy for effective channels, it measures \emph{semantic} diversity in embedding space rather than \emph{task-relevant} information diversity. As shown in Section~\ref{sec:kstar_analysis}, the decomposition into $K^*_c$ and $K^*_w$ reveals that not all diversity is beneficial, only diversity among correct reasoning paths reliably improves performance. Furthermore, $K^*$ depends on the choice of embedding model, and its correlation with accuracy varies across task types (stronger for reasoning tasks, weaker for knowledge-intensive tasks). Developing task-adaptive diversity metrics remains an open problem.

\vspace{-0.2cm}
\paragraph{Empirical Scope.}
Our experiments focus on 7B-8B scale open-weight models across seven benchmarks. Whether the diversity-over-scale principle extends to larger models, closed-source APIs, or more complex agentic workflows (e.g., tool use, long-horizon planning) requires further investigation. Additionally, our analysis considers vote and debate mechanisms; other coordination protocols may exhibit different scaling behaviors.

\bibliography{main}

@article{ton2024understanding,
  title={Understanding chain-of-thought in llms through information theory},
  author={Ton, Jean-Francois and Taufiq, Muhammad Faaiz and Liu, Yang},
  journal={arXiv preprint arXiv:2411.11984},
  year={2024}
}

@article{gan2025rethinking,
  title={Rethinking external slow-thinking: From snowball errors to probability of correct reasoning},
  author={Gan, Zeyu and Liao, Yun and Liu, Yong},
  journal={arXiv preprint arXiv:2501.15602},
  year={2025}
}

@article{riedl2025emergent,
  title={Emergent coordination in multi-agent language models},
  author={Riedl, Christoph},
  journal={arXiv preprint arXiv:2510.05174},
  year={2025}
}

@book{chang2025multi,
  title={Multi-LLM Agent Collaborative Intelligence: The Path to Artificial General Intelligence},
  author={Chang, Edward Y},
  year={2025},
  publisher={Edward Y. Chang}
}

@article{choi2025debate,
  title={Debate or Vote: Which Yields Better Decisions in Multi-Agent Large Language Models?},
  author={Choi, Hyeong Kyu and Zhu, Xiaojin and Li, Sharon},
  journal={arXiv preprint arXiv:2508.17536},
  year={2025}
}

@article{wu2025hidden,
  title={The hidden strength of disagreement: Unraveling the consensus-diversity tradeoff in adaptive multi-agent systems},
  author={Wu, Zengqing and Ito, Takayuki},
  journal={arXiv preprint arXiv:2502.16565},
  year={2025}
}

@article{zhang2024diversity,
  title={Diversity empowers intelligence: Integrating expertise of software engineering agents},
  author={Zhang, Kexun and Yao, Weiran and Liu, Zuxin and Feng, Yihao and Liu, Zhiwei and Murthy, Rithesh and Lan, Tian and Li, Lei and Lou, Renze and Xu, Jiacheng and others},
  journal={arXiv preprint arXiv:2408.07060},
  year={2024}
}

@article{wang2024mixture,
  title={Mixture-of-agents enhances large language model capabilities},
  author={Wang, Junlin and Wang, Jue and Athiwaratkun, Ben and Zhang, Ce and Zou, James},
  journal={arXiv preprint arXiv:2406.04692},
  year={2024}
}

@article{tang2025importance,
  title={On the Importance of Task Complexity in Evaluating LLM-Based Multi-Agent Systems},
  author={Tang, Bohan and Liang, Huidong and Jiang, Keyue and Dong, Xiaowen},
  journal={arXiv preprint arXiv:2510.04311},
  year={2025}
}

@article{yuen2025intrinsic,
  title={Intrinsic Memory Agents: Heterogeneous Multi-Agent LLM Systems through Structured Contextual Memory},
  author={Yuen, Sizhe and Medina, Francisco Gomez and Su, Ting and Du, Yali and Sobey, Adam J},
  journal={arXiv preprint arXiv:2508.08997},
  year={2025}
}

@article{samuel2024personagym,
  title={Personagym: Evaluating persona agents and llms},
  author={Samuel, Vinay and Zou, Henry Peng and Zhou, Yue and Chaudhari, Shreyas and Kalyan, Ashwin and Rajpurohit, Tanmay and Deshpande, Ameet and Narasimhan, Karthik and Murahari, Vishvak},
  journal={arXiv preprint arXiv:2407.18416},
  year={2024}
}

@article{taillandier2025integrating,
  title={Integrating llm in agent-based social simulation: Opportunities and challenges},
  author={Taillandier, Patrick and Zucker, Jean Daniel and Grignard, Arnaud and Gaudou, Benoit and Huynh, Nghi Quang and Drogoul, Alexis},
  journal={arXiv preprint arXiv:2507.19364},
  year={2025}
}

@article{kim2025sciencescalingagentsystems,
      title={Towards a Science of Scaling Agent Systems}, 
      author={Yubin Kim and Ken Gu and Chanwoo Park and Chunjong Park and Samuel Schmidgall and A. Ali Heydari and Yao Yan and Zhihan Zhang and Yuchen Zhuang and Mark Malhotra and Paul Pu Liang and Hae Won Park and Yuzhe Yang and Xuhai Xu and Yilun Du and Shwetak Patel and Tim Althoff and Daniel McDuff and Xin Liu},
      year={2025},
      journal={arXiv preprint arXiv:2512.08296}, 
}

@article{allenai:arc,
      author    = {Peter Clark  and Isaac Cowhey and Oren Etzioni and Tushar Khot and
                    Ashish Sabharwal and Carissa Schoenick and Oyvind Tafjord},
      title     = {Think you have Solved Question Answering? Try ARC, the AI2 Reasoning Challenge},
      journal   = {arXiv:1803.05457v1},
      year      = {2018},
}

@article{cobbe2021gsm8k,
  title={Training Verifiers to Solve Math Word Problems},
  author={Cobbe, Karl and Kosaraju, Vineet and Bavarian, Mohammad and Chen, Mark and Jun, Heewoo and Kaiser, Lukasz and Plappert, Matthias and Tworek, Jerry and Hilton, Jacob and Nakano, Reiichiro and Hesse, Christopher and Schulman, John},
  journal={arXiv preprint arXiv:2110.14168},
  year={2021}
}

@InProceedings{ai2:winogrande,
title = {WinoGrande: An Adversarial Winograd Schema Challenge at Scale},
authors={Keisuke, Sakaguchi and Ronan, Le Bras and Chandra, Bhagavatula and Yejin, Choi
},
year={2019}
}

@article{hendryckstest2021,
  title={Measuring Massive Multitask Language Understanding},
  author={Dan Hendrycks and Collin Burns and Steven Basart and Andy Zou and Mantas Mazeika and Dawn Song and Jacob Steinhardt},
  journal={Proceedings of the International Conference on Learning Representations (ICLR)},
  year={2021}
}

@article{hendrycks2021ethics,
  title={Aligning AI With Shared Human Values},
  author={Dan Hendrycks and Collin Burns and Steven Basart and Andrew Critch and Jerry Li and Dawn Song and Jacob Steinhardt},
  journal={Proceedings of the International Conference on Learning Representations (ICLR)},
  year={2021}
}

@misc{lin2022truthfulqameasuringmodelsmimic,
      title={TruthfulQA: Measuring How Models Mimic Human Falsehoods}, 
      author={Stephanie Lin and Jacob Hilton and Owain Evans},
      year={2022},
      eprint={2109.07958},
      archivePrefix={arXiv},
      primaryClass={cs.CL},
}

@inproceedings{zellers2019hellaswag,
    title={HellaSwag: Can a Machine Really Finish Your Sentence?},
    author={Zellers, Rowan and Holtzman, Ari and Bisk, Yonatan and Farhadi, Ali and Choi, Yejin},
    booktitle ={Proceedings of the 57th Annual Meeting of the Association for Computational Linguistics},
    year={2019}
}

@article{qwen2.5,
    title = {Qwen2.5: A Party of Foundation Models},
    author = {{Qwen Team}},
    journal = {arXiv preprint arXiv:2412.15115},
    year = {2024}
}

@misc{jiang2023mistral7b,
      title={Mistral 7B}, 
      author={Albert Q. Jiang and Alexandre Sablayrolles and Arthur Mensch and Chris Bamford and Devendra Singh Chaplot and Diego de las Casas and Florian Bressand and Gianna Lengyel and Guillaume Lample and Lucile Saulnier and Lélio Renard Lavaud and Marie-Anne Lachaux and Pierre Stock and Teven Le Scao and Thibaut Lavril and Thomas Wang and Timothée Lacroix and William El Sayed},
      year={2023},
      eprint={2310.06825},
      archivePrefix={arXiv},
      primaryClass={cs.CL},
}

@article{grattafiori2024llama3herdmodels,
      title={The Llama 3 Herd of Models},
      author={Grattafiori, Aaron and others},
      year={2024},
      journal={arXiv preprint arXiv:2407.21783},
}

@article{chen2024llm,
      title={Are More LLM Calls All You Need? Towards Scaling Laws of Compound Inference Systems}, 
      author={Lingjiao Chen and Jared Quincy Davis and Boris Hanin and Peter Bailis and Ion Stoica and Matei Zaharia and James Zou},
      year={2024},
      eprint={2403.02419},
      journal={arXiv preprint arXiv:2403.02419},
}

@article{wang2024a,
  title = {A Survey on Large Language Model Based Autonomous Agents},
  author = {Wang, Lei and Ma, Chen and Feng, Xueyang and Zhang, Zeyu and Yang, Hao and Zhang, Jingsen and Chen, Zhiyuan and Tang, Jiakai and Chen, Xu and Lin, Yankai and Zhao, Wayne Xin and Wei, Zhewei and Wen, Jirong},
  year = {2024},
  journal = {Frontiers of Computer Science},
  volume = {18},
  number = {6},
  pages = {186345},
  doi = {10.1007/s11704-024-40231-1},
}

@article{xi2023risepotentiallargelanguage,
      title={The Rise and Potential of Large Language Model Based Agents: A Survey}, 
      author={Zhiheng Xi and Wenxiang Chen and Xin Guo and Wei He and Yiwen Ding and Boyang Hong and Ming Zhang and Junzhe Wang and Senjie Jin and Enyu Zhou and Rui Zheng and Xiaoran Fan and Xiao Wang and Limao Xiong and Yuhao Zhou and Weiran Wang and Changhao Jiang and Yicheng Zou and Xiangyang Liu and Zhangyue Yin and Shihan Dou and Rongxiang Weng and Wensen Cheng and Qi Zhang and Wenjuan Qin and Yongyan Zheng and Xipeng Qiu and Xuanjing Huang and Tao Gui},
      year={2023},
      eprint={2309.07864},
      journal={arXiv preprint arXiv:2309.07864},
}

@article{du2023improvingfactualityreasoninglanguage,
      title={Improving Factuality and Reasoning in Language Models through Multiagent Debate}, 
      author={Yilun Du and Shuang Li and Antonio Torralba and Joshua B. Tenenbaum and Igor Mordatch},
      year={2023},
      eprint={2305.14325},
      journal={arXiv preprint arXiv:2305.14325},
}

@article{hong2024metagptmetaprogrammingmultiagent,
      title={MetaGPT: Meta Programming for A Multi-Agent Collaborative Framework}, 
      author={Sirui Hong and Mingchen Zhuge and Jiaqi Chen and Xiawu Zheng and Yuheng Cheng and Ceyao Zhang and Jinlin Wang and Zili Wang and Steven Ka Shing Yau and Zijuan Lin and Liyang Zhou and Chenyu Ran and Lingfeng Xiao and Chenglin Wu and Jürgen Schmidhuber},
      year={2024},
      eprint={2308.00352},
      journal={arXiv preprint arXiv:2308.00352},
}

@inproceedings{
dang2025multiagent,
title={Multi-Agent Collaboration via Evolving Orchestration},
author={Yufan Dang and Chen Qian and Xueheng Luo and Jingru Fan and Zihao Xie and Ruijie Shi and Weize Chen and Cheng Yang and Xiaoyin Che and Ye Tian and Xuantang Xiong and Lei Han and Zhiyuan Liu and Maosong Sun},
booktitle={The Thirty-ninth Annual Conference on Neural Information Processing Systems},
year={2025},
}

@article{li2024agentsneed,
      title={More Agents Is All You Need}, 
      author={Junyou Li and Qin Zhang and Yangbin Yu and Qiang Fu and Deheng Ye},
      year={2024},
      eprint={2402.05120},
      journal={arXiv preprint arXiv:2402.05120},
}

@inproceedings{
qian2025scaling,
title={Scaling Large Language Model-based Multi-Agent Collaboration},
author={Chen Qian and Zihao Xie and YiFei Wang and Wei Liu and Kunlun Zhu and Hanchen Xia and Yufan Dang and Zhuoyun Du and Weize Chen and Cheng Yang and Zhiyuan Liu and Maosong Sun},
booktitle={The Thirteenth International Conference on Learning Representations},
year={2025},
}

@article{guo2024large,
      title={Large Language Model based Multi-Agents: A Survey of Progress and Challenges}, 
      author={Taicheng Guo and Xiuying Chen and Yaqi Wang and Ruidi Chang and Shichao Pei and Nitesh V. Chawla and Olaf Wiest and Xiangliang Zhang},
      year={2024},
      eprint={2402.01680},
      journal={arXiv preprint arXiv:2402.01680}, 
}

@inproceedings{
cemri2025why,
title={Why Do Multi-Agent {LLM} Systems Fail?},
author={Mert Cemri and Melissa Z Pan and Shuyi Yang and Lakshya A Agrawal and Bhavya Chopra and Rishabh Tiwari and Kurt Keutzer and Aditya Parameswaran and Dan Klein and Kannan Ramchandran and Matei Zaharia and Joseph E. Gonzalez and Ion Stoica},
booktitle={The Thirty-ninth Annual Conference on Neural Information Processing Systems Datasets and Benchmarks Track},
year={2025},
}

@inproceedings{
wu2024autogen,
title={AutoGen: Enabling Next-Gen {LLM} Applications via Multi-Agent Conversations},
author={Qingyun Wu and Gagan Bansal and Jieyu Zhang and Yiran Wu and Beibin Li and Erkang Zhu and Li Jiang and Xiaoyun Zhang and Shaokun Zhang and Jiale Liu and Ahmed Hassan Awadallah and Ryen W White and Doug Burger and Chi Wang},
booktitle={First Conference on Language Modeling},
year={2024},
}

@inproceedings{li-etal-2024-improving-multi,
    title = "Improving Multi-Agent Debate with Sparse Communication Topology",
    author = "Li, Yunxuan  and
      Du, Yibing  and
      Zhang, Jiageng  and
      Hou, Le  and
      Grabowski, Peter  and
      Li, Yeqing  and
      Ie, Eugene",
    editor = "Al-Onaizan, Yaser  and
      Bansal, Mohit  and
      Chen, Yun-Nung",
    booktitle = "Findings of the Association for Computational Linguistics: EMNLP 2024",
    month = nov,
    year = "2024",
    address = "Miami, Florida, USA",
    publisher = "Association for Computational Linguistics",
    doi = "10.18653/v1/2024.findings-emnlp.427",
    pages = "7281--7294",
}

@inproceedings{
zhao2025sirius,
title={SiriuS: Self-improving Multi-agent Systems via Bootstrapped Reasoning},
author={Wanjia Zhao and Mert Yuksekgonul and Shirley Wu and James Zou},
booktitle={The Thirty-ninth Annual Conference on Neural Information Processing Systems},
year={2025},
}

@inproceedings{camel,
  author = {Li, Guohao and Al Kader Hammoud, Hasan Abed and Itani, Hani and Khizbullin, Dmitrii and Ghanem, Bernard},
  title = {CAMEL: Communicative Agents for ``Mind'' Exploration of Large Language Model Society},
  year = {2023},
  booktitle = {Proceedings of the 37th International Conference on Neural Information Processing Systems},
  series = {NeurIPS '23}
}

@inproceedings{10.5555/3692070.3693020,
  author = {Khan, Akbir and Hughes, John and Valentine, Dan and Ruis, Laura and Sachan, Kshitij and Radhakrishnan, Ansh and Grefenstette, Edward and Bowman, Samuel R. and Rockt\"{a}schel, Tim and Perez, Ethan},
  title = {Debating with More Persuasive {LLMs} Leads to More Truthful Answers},
  year = {2024},
  booktitle = {Proceedings of the 41st International Conference on Machine Learning},
  series = {ICML'24}
}

@article{wang2023selfconsistencyimproveschainthought,
      title={Self-Consistency Improves Chain of Thought Reasoning in Language Models}, 
      author={Xuezhi Wang and Jason Wei and Dale Schuurmans and Quoc Le and Ed Chi and Sharan Narang and Aakanksha Chowdhery and Denny Zhou},
      year={2023},
      eprint={2203.11171},
      journal={arXiv preprint arXiv:2203.11171}, 
}

@misc{lee2025nvembedimprovedtechniquestraining,
      title={NV-Embed: Improved Techniques for Training LLMs as Generalist Embedding Models}, 
      author={Chankyu Lee and Rajarshi Roy and Mengyao Xu and Jonathan Raiman and Mohammad Shoeybi and Bryan Catanzaro and Wei Ping},
      year={2025},
      eprint={2405.17428},
      archivePrefix={arXiv},
      primaryClass={cs.CL},
}

@article{wei2022chain,
  title={Chain-of-thought prompting elicits reasoning in large language models},
  author={Wei, Jason and Wang, Xuezhi and Schuurmans, Dale and Maeda, Maarten and Chi, Ed and Xia, Sharan and Le, Quoc and Zhou, Denny},
  journal={Advances in Neural Information Processing Systems},
  volume={35},
  pages={24824--24837},
  year={2022}
}

@article{achiam2023gpt,
  title={GPT-4 Technical Report},
  author={Achiam, Josh and Adler, Steven and Agarwal, Sandhini and Ahmad, Lama and Akkaya, Ilge and Aleman, Florencia Leoni and Almeida, Diogo and Altenschmidt, Janko and Altman, Sam and Anadkat, Shyamal and others},
  journal={arXiv preprint arXiv:2303.08774},
  year={2023}
}

@article{huang2023large,
  title={Large Language Models Cannot Self-Correct Reasoning Yet},
  author={Huang, Jie and Gu, Xinyun and Chen, Le and Han, Jiawei and Kraska, Tim},
  journal={arXiv preprint arXiv:2310.01798},
  year={2023}
}

@inproceedings{qian2024chatdev,
  title={ChatDev: Communicative Agents for Software Development},
  author={Qian, Chen and Liu, Wei and Liu, Hongzhang and Chen, Nuo and Dang, Yufan and Li, Jiahao and Yang, Cheng and Chen, Weize and Su, Yusheng and Cong, Xin and Xu, Juyuan and Li, Dahai and Liu, Zhiyuan and Sun, Maosong},
  booktitle={Proceedings of the 62nd Annual Meeting of the Association for Computational Linguistics (Volume 1: Long Papers)},
  pages={15174--15186},
  year={2024},
  publisher={Association for Computational Linguistics}
}
\bibliographystyle{icml2026}

\newpage
\appendix
\onecolumn

\section{Proofs and Technical Details}
\label{sec:appendix}

This appendix provides full proofs and technical details omitted from the main text.
Section~\ref{app:prelim} reviews standard information-theoretic identities.
Section~\ref{app:budget} proves Theorem~\ref{thm:finite_budget}.
Sections~\ref{app:parallel}--\ref{app:sequential} derive upper bounds for common MAS workflows.
Section~\ref{app:lowerbound} presents the evidence-bits coverage model and proves
Theorem~\ref{thm:lb_saturated} and Corollary~\ref{cor:lb} from the main text.
Section~\ref{app:kstar} proves basic properties of the effective channel count $K^*$.

\subsection{Information-Theoretic Preliminaries}
\label{app:prelim}

We recall standard definitions and lemmas from information theory.

\begin{definition}[Conditional Mutual Information]
For random variables $A,B,C$,
\begin{equation}
I(A;B\mid C)
= H(A\mid C)-H(A\mid B,C)
= H(B\mid C)-H(B\mid A,C).
\end{equation}
\end{definition}

\begin{lemma}[Chain Rule for Mutual Information]
\label{lem:chain-rule}
For random variables $X,Y_1,\ldots,Y_n,Z$,
\begin{equation}
I(Y_1,\ldots,Y_n;X\mid Z)
= \sum_{i=1}^n I(Y_i;X\mid Z,Y_{<i}).
\end{equation}
\end{lemma}

\begin{lemma}[Data Processing Inequality]
\label{lem:dpi}
If $A\to B\to C$ forms a Markov chain, then $I(A;C)\le I(A;B)$.
\end{lemma}

\begin{lemma}[Incremental Information as Entropy Difference]
\label{lem:delta-entropy}
Let $\Delta_i := I(Z_i;Y\mid X,Z_{<i})$. Then
\begin{equation}
\Delta_i
= H(Y\mid X,Z_{<i}) - H(Y\mid X,Z_{\le i}),
\end{equation}
where $Z_{\le i}=(Z_1,\ldots,Z_i)$.
\end{lemma}

\begin{proof}
By the definition of conditional mutual information,
\begin{align}
\Delta_i
&= I(Z_i;Y\mid X,Z_{<i}) \\
&= H(Y\mid X,Z_{<i}) - H(Y\mid X,Z_{<i},Z_i) \\
&= H(Y\mid X,Z_{<i}) - H(Y\mid X,Z_{\le i}).\qedhere
\end{align}
\end{proof}

\subsection{Finite Information Budget (Upper Bound)}
\label{app:budget}

For completeness, the total information an MAS can extract is always upper-bounded by the intrinsic task uncertainty.

\begin{theorem}[Finite Information Budget]
\label{thm:finite_budget_app}
For any MAS transcript $Z_{1:n}$,
\begin{equation}
I(Z_{1:n};Y\mid X)\le H(Y\mid X).
\end{equation}
Moreover, writing $I(Z_{1:n};Y\mid X)=\sum_{i=1}^n \Delta_i$ with $\Delta_i:=I(Z_i;Y\mid X,Z_{<i})$, we have $\sum_{i=1}^n \Delta_i\le H(Y\mid X)$ and $\Delta_i\to 0$ as $i\to\infty$.
\end{theorem}

\begin{proof}
By the definition of conditional mutual information,
\begin{equation}
I(Z_{1:n};Y\mid X)
= H(Y\mid X) - H(Y\mid X,Z_{1:n})
\le H(Y\mid X),
\end{equation}
since conditional entropy is nonnegative.
By Lemma~\ref{lem:chain-rule},
\begin{equation}
I(Z_{1:n};Y\mid X)
= \sum_{i=1}^n I(Z_i;Y\mid X,Z_{<i})
= \sum_{i=1}^n \Delta_i.
\end{equation}
Because $\Delta_i\ge 0$ and the partial sums are uniformly bounded by $H(Y\mid X)$, we must have $\Delta_i\to 0$ as $i\to\infty$.
\end{proof}

\subsection{Parallel Voting: Assumptions and Upper Bounds}
\label{app:parallel}

This section derives the parallel-voting upper bounds used in the main text.
The key message is that repeated sampling from the same configuration produces redundant evidence.

\subsubsection{Conditional Independence for Parallel Sampling}
\label{app:parallel-ci}

\begin{assumption}[Conditional Independence for Parallel Sampling (All Types)]
\label{assump:parallel-ci}
Consider a parallel MAS with agent configuration types $b(i)\in\mathcal{B}$.
There exist channels $\{K_b(\cdot\mid x,y)\}_{b\in\mathcal{B}}$ such that, for every $i$,
\begin{equation}
P(Z_i=z\mid X=x,Y=y, Z_{<i}) \;=\; K_{b(i)}(z\mid x,y),
\end{equation}
and the outputs are mutually independent conditioned on $(X,Y)$:
\begin{equation}
P(Z_{1:n}\mid X,Y) \;=\; \prod_{i=1}^n P(Z_i\mid X,Y, b(i)).
\end{equation}
\end{assumption}

Define the single-call information for type $b$:
\begin{equation}
I_b := I(Z^{(b)};Y\mid X),
\label{eq:Ib_def}
\end{equation}
where $Z^{(b)}$ denotes one output from type $b$ in isolation.

\subsubsection{A Redundancy Identity}
\label{app:parallel-identity}

\begin{lemma}[Three-Way Mutual Information Decomposition]
\label{lem:three-way}
For any random variables $A,B,C,D$,
\begin{equation}
I(A;B\mid C,D)
= I(A;B\mid C) + I(A;D\mid B,C) - I(A;D\mid C).
\end{equation}
\end{lemma}

\begin{proof}
Apply chain rule in two ways:
\begin{align}
I(A;B,D\mid C) &= I(A;B\mid C) + I(A;D\mid B,C), \\
I(A;B,D\mid C) &= I(A;D\mid C) + I(A;B\mid C,D).
\end{align}
Equating and rearranging yields the claim.
\end{proof}

\begin{corollary}[Incremental Gain under Parallel Sampling]
\label{cor:parallel-identity}
With $A=Z_i$, $B=Y$, $C=X$, $D=Z_{<i}$,
\begin{equation}
I(Z_i;Y\mid X,Z_{<i})
=
I(Z_i;Y\mid X)
+
I(Z_i;Z_{<i}\mid X,Y)
-
I(Z_i;Z_{<i}\mid X).
\label{eq:parallel-identity}
\end{equation}
Under Assumption~\ref{assump:parallel-ci}, $I(Z_i;Z_{<i}\mid X,Y)=0$, hence
\begin{equation}
I(Z_i;Y\mid X,Z_{<i})
= I(Z_i;Y\mid X) - I(Z_i;Z_{<i}\mid X)
\le I(Z_i;Y\mid X).
\label{eq:delta_le_marginal}
\end{equation}
\end{corollary}

\noindent This formalizes redundancy: previous outputs can only reduce the new information.

\paragraph{Implication: redundancy controls early saturation.}
Upper bounds identify \emph{what} limits the total information gain.
To explain \emph{when} saturation occurs, consider the incremental contribution
$\Delta_i := I(Z_i;Y\mid X,Z_{<i})$.
Eq.~\eqref{eq:delta_le_marginal} provides an explicit decomposition:
\begin{equation}
\Delta_i
= I(Z_i;Y\mid X) - I(Z_i;Z_{<i}\mid X),
\label{eq:redundancy_identity_main_app}
\end{equation}
where the \emph{redundancy term} $I(Z_i;Z_{<i}\mid X)$ quantifies how much the $i$-th output overlaps with previous outputs.
Thus, early saturation arises when repeated calls increase $I(Z_i;Z_{<i}\mid X)$, leaving little additional evidence to accumulate.
Homogeneous agents typically induce large redundancy due to similar reasoning trajectories, while heterogeneity mitigates overlap and sustains $\Delta_i$.

Since $I(Z_i;Y\mid X,Z_{<i})\ge 0$, the identity also implies
$I(Z_i;Z_{<i}\mid X)\le I(Z_i;Y\mid X)$ under Assumption~\ref{assump:parallel-ci}.

\subsubsection{Homogeneous Parallel Bound}
\label{app:homog-parallel}

\begin{proposition}[Homogeneous Parallel Upper Bound]
\label{prop:homog-parallel}
Assume $m$ parallel samples from a single type $b$ under Assumption~\ref{assump:parallel-ci}.
Then
\begin{equation}
I(Z_{1:m};Y\mid X)\le H(Y\mid X) \;\wedge\; mI_b.
\label{eq:homog-parallel}
\end{equation}
\end{proposition}

\begin{proof}
By chain rule,
\begin{equation}
I(Z_{1:m};Y\mid X)
=\sum_{i=1}^m I(Z_i;Y\mid X,Z_{<i}).
\end{equation}
Using Eq.~\eqref{eq:delta_le_marginal} and $I(Z_i;Y\mid X)=I_b$ for all $i$,
\begin{equation}
I(Z_{1:m};Y\mid X)\le \sum_{i=1}^m I_b = mI_b.
\end{equation}
The finite budget further implies $I(Z_{1:m};Y\mid X)\le H(Y\mid X)$.
Combining yields Eq.~\eqref{eq:homog-parallel}.
\end{proof}

\subsubsection{Heterogeneous Parallel Bound}
\label{app:heterog-parallel}

\begin{theorem}[Heterogeneous Parallel Upper Bound]
\label{thm:heterog-parallel}
Consider parallel voting with configuration types $\mathcal{B}$.
Let type $b$ be sampled $m_b$ times, with total $n=\sum_{b\in\mathcal{B}} m_b$.
Then
\begin{equation}
I(Z_{1:n};Y\mid X)
\;\le\;
H(Y\mid X) \;\wedge\; \sum_{b\in\mathcal{B}} m_b\, I_b.
\label{eq:heterog-parallel}
\end{equation}
\end{theorem}

\begin{proof}
Apply the chain rule:
\begin{equation}
I(Z_{1:n};Y\mid X)=\sum_{i=1}^n I(Z_i;Y\mid X,Z_{<i}).
\end{equation}
By Eq.~\eqref{eq:delta_le_marginal}, each term is bounded by $I(Z_i;Y\mid X)=I_{b(i)}$.
Summing over steps grouped by type gives $\sum_{b\in\mathcal{B}} m_b\, I_b$.
The finite budget gives the minimum with $H(Y\mid X)$.
\end{proof}

\subsection{Sequential Pipelines and Debate: Upper Bounds}
\label{app:sequential}

In sequential settings, each output conditions on the interaction history.
This invalidates conditional independence, but the chain rule remains valid.

\subsubsection{Maximal Per-Step Contribution}
\label{app:imax}

Define the maximal incremental contribution for agent configuration type $b$:
\begin{equation}
I_b^{\max}
:=
\sup_{z_{<i}}
I(Z_i;Y\mid X,Z_{<i}=z_{<i}, b(i)=b).
\label{eq:imax}
\end{equation}

\begin{proposition}[Sequential Pipeline Upper Bound]
\label{prop:sequential}
For any sequential MAS with $n$ steps,
\begin{equation}
I(Z_{1:n};Y\mid X)
\;\le\;
H(Y\mid X) \;\wedge\; \sum_{i=1}^n I_{b(i)}^{\max}.
\label{eq:sequential-bound}
\end{equation}
\end{proposition}

\begin{proof}
By chain rule,
\begin{equation}
I(Z_{1:n};Y\mid X)=\sum_{i=1}^n I(Z_i;Y\mid X,Z_{<i}).
\end{equation}
For each $i$, by definition of $I_{b(i)}^{\max}$ we have
$I(Z_i;Y\mid X,Z_{<i})\le I_{b(i)}^{\max}$.
Summing yields the stated bound, and the finite budget gives the minimum with $H(Y\mid X)$.
\end{proof}

\paragraph{Debate.}
Two-agent debate is a special case of sequential interaction and inherits the same ceiling $H(Y\mid X)$.
This formalizes why debate cannot systematically improve over voting if agents remain redundant.

\subsection{Lower Bound via Independent Evidence-Bits Coverage}
\label{app:lowerbound}

This section formalizes the ``effective channels'' view used in the main text.
It proves Theorem~\ref{thm:lb_saturated} (geometric contraction of the residual uncertainty)
and Corollary~\ref{cor:lb} (the saturated lower bound in expectation), which together imply a characteristic rapid-then-saturating improvement curve $1-e^{-\alpha K}$ emphasized in Eq.~\eqref{eq:info_curve} of the main text.

\subsubsection{Evidence Bits Model}
\label{app:evidence-bits}

\begin{assumption}[Independent Evidence Bits]
\label{assump:evidence-bits}
There exist latent variables $U=(U_1,\ldots,U_M)$ such that:
\begin{enumerate}[leftmargin=1.2em]
\item (\textbf{Sufficiency}) $H(Y\mid X,U)=0$.
\item (\textbf{Conditional independence}) $U_1,\ldots,U_M$ are independent conditioned on $X$.
\item (\textbf{Matching uncertainty scale}) $H(U\mid X)=H(Y\mid X)$.
\end{enumerate}
\end{assumption}

Assumption~\ref{assump:evidence-bits}(iii) calibrates the latent ``evidence bits'' to exactly match the
intrinsic task uncertainty. In particular, recovering all evidence bits eliminates residual uncertainty about $Y$.

\subsubsection{Fractional Coverage by Effective Channels}
\label{app:coverage}

\begin{assumption}[Fractional Evidence Coverage]
\label{assump:coverage}
Let $\tilde Z_{1:K}$ denote $K$ effective channels extracted from an MAS transcript.
For each evidence bit $U_j$ and each channel $k\in\{1,\ldots,K\}$, define a Bernoulli indicator $C_{j,k}\in\{0,1\}$:
$C_{j,k}=1$ means channel $k$ reveals $U_j$.
Assume:
\begin{enumerate}[leftmargin=1.2em]
\item $\mathbb{P}(C_{j,k}=1)=\alpha$ for some fixed $\alpha\in(0,1)$.
\item For each fixed $j$, $\{C_{j,k}\}_{k=1}^K$ are independent.
\item If $\exists k$ such that $C_{j,k}=1$, then $H(U_j\mid X,\tilde Z_{1:K})=0$.
\end{enumerate}
\end{assumption}

Assumption~\ref{assump:coverage} is a minimal complementarity model: each \emph{new effective channel} has a constant
probability $\alpha$ of covering any remaining evidence bit, independently across channels.

\subsubsection{Residual Contraction and Saturated Lower Bound}
\label{app:residual}

\begin{lemma}[Expected Geometric Decay of Residual Uncertainty]
\label{lem:residual}
Under Assumptions~\ref{assump:evidence-bits} and \ref{assump:coverage},
\begin{equation}
\EE\big[H(Y\mid X,\tilde Z_{1:K})\big]
\;\le\;
(1-\alpha)^K\,H(Y\mid X).
\label{eq:residual-decay-exp}
\end{equation}
\end{lemma}

\begin{proof}
By Assumption~\ref{assump:evidence-bits}(i), $Y$ is a function of $(X,U)$, hence
\begin{equation}
H(Y\mid X,\tilde Z_{1:K})
\le
H(U\mid X,\tilde Z_{1:K}).
\end{equation}
Subadditivity of conditional entropy yields
\begin{equation}
H(U\mid X,\tilde Z_{1:K})
\le
\sum_{j=1}^M H(U_j\mid X,\tilde Z_{1:K}).
\end{equation}

Fix $j$. If $U_j$ is revealed by at least one effective channel, then by
Assumption~\ref{assump:coverage}(iii),
$H(U_j\mid X,\tilde Z_{1:K})=0$; otherwise,
$H(U_j\mid X,\tilde Z_{1:K})\le H(U_j\mid X)$.
The probability that $U_j$ is not revealed by any of the $K$ channels is $(1-\alpha)^K$
by Assumption~\ref{assump:coverage}(ii). Therefore,
\begin{equation}
\EE\big[H(U_j\mid X,\tilde Z_{1:K})\big]
\le
(1-\alpha)^K\,H(U_j\mid X).
\end{equation}
Summing over $j$ gives
\begin{equation}
\EE\big[H(U\mid X,\tilde Z_{1:K})\big]
\le
(1-\alpha)^K \sum_{j=1}^M H(U_j\mid X).
\end{equation}
Finally, by Assumption~\ref{assump:evidence-bits}(ii),
$H(U\mid X)=\sum_{j=1}^M H(U_j\mid X)$, and by (iii) $H(U\mid X)=H(Y\mid X)$.
Combining completes the proof.
\end{proof}

\begin{theorem}[Geometric Contraction with Effective Channels]
\label{thm:lb_saturated}
Under Assumptions~\ref{assump:evidence-bits} and~\ref{assump:coverage},
\begin{equation}
H(Y\mid X) - \EE\big[I(\tilde Z_{1:K};Y\mid X)\big]
~=~
\EE\big[H(Y\mid X,\tilde Z_{1:K})\big]
\;\le\;
(1-\alpha)^K\,H(Y\mid X).
\label{eq:lower_bound_main_app}
\end{equation}
Consequently,
\begin{equation}
\boxed{
\begin{split}
&H(Y\mid X) - \EE\!\left[I(\tilde Z_{1:K};Y\mid X)\right] \\
&\;\le\; (1-\alpha)^K \, H(Y\mid X)
\;\le\; e^{-\alpha K} \, H(Y\mid X).
\end{split}
}
\label{eq:lower_bound_main}
\end{equation}
Equivalently, the \emph{normalized residual} satisfies
$\EE[H(Y\mid X,\tilde Z_{1:K})]/H(Y\mid X)\le (1-\alpha)^K \le e^{-\alpha K}$.
\end{theorem}

\begin{proof}
By definition,
$I(\tilde Z_{1:K};Y\mid X)=H(Y\mid X)-H(Y\mid X,\tilde Z_{1:K})$.
Taking expectations yields
\begin{equation}
H(Y\mid X) - \EE[I(\tilde Z_{1:K};Y\mid X)]
=
\EE[H(Y\mid X,\tilde Z_{1:K})].
\end{equation}
Apply Lemma~\ref{lem:residual} to obtain \eqref{eq:lower_bound_main_app}.
The exponential form follows from $(1-\alpha)^K\le e^{-\alpha K}$.
\end{proof}

\begin{corollary}[Saturated Lower Bound (in Expectation)]
\label{cor:lb}
Under the same assumptions,
\begin{equation}
\boxed{
\EE\big[I(\tilde Z_{1:K};Y\mid X)\big]
\;\ge\;
H(Y\mid X)\Big(1-(1-\alpha)^K\Big)
\;\ge\;
H(Y\mid X)\Big(1-e^{-\alpha K}\Big).
}
\label{eq:lowerbound-exp}
\end{equation}
\end{corollary}

\begin{proof}
Rearrange the identity
$\EE[I(\tilde Z_{1:K};Y\mid X)]
=
H(Y\mid X)-\EE[H(Y\mid X,\tilde Z_{1:K})]$
and apply Lemma~\ref{lem:residual}. The exponential form follows from
$(1-\alpha)^K\le e^{-\alpha K}$.
\end{proof}

\subsubsection{Heterogeneity Advantage as an $\alpha K$ Comparison}
\label{app:heterog_homog_gap}

This subsection provides a formal underpinning for the main-text comparison
(Corollary~\ref{cor:heterog_homog_gap_B}): heterogeneity improves expected recoverable information whenever it increases
the effective evidence term $\alpha K$.

\begin{lemma}[Monotonicity in $\alpha K$]
\label{lem:monotone_alphaK}
Define $f(t):=1-e^{-t}$ for $t\ge 0$.
Then for any $t_1,t_2\ge 0$, $t_2>t_1$ implies $f(t_2)>f(t_1)$.
\end{lemma}

\begin{proof}
$f'(t)=e^{-t}>0$ for all $t\ge 0$, hence $f$ is strictly increasing.
\end{proof}

\begin{corollary}[Heterogeneity Advantage from Corollary~\ref{cor:lb}]
\label{cor:heterog_homog_gap}
Consider two designs summarized by $(K_{\mathrm{homog}},\alpha_{\mathrm{homog}})$ and $(K_{\mathrm{heterog}},\alpha_{\mathrm{heterog}})$
under Assumptions~\ref{assump:evidence-bits}--\ref{assump:coverage}.
By Corollary~\ref{cor:lb}, the lower bounds on recoverable information for the two designs are:
\begin{align}
\EE[I_{\mathrm{heterog}}]
&\ge H(Y\mid X)\bigl(1-e^{-\alpha_{\mathrm{heterog}}K_{\mathrm{heterog}}}\bigr), \label{eq:lb_heterog_app}\\
\EE[I_{\mathrm{homog}}]
&\ge H(Y\mid X)\bigl(1-e^{-\alpha_{\mathrm{homog}}K_{\mathrm{homog}}}\bigr). \label{eq:lb_homog_app}
\end{align}
When $\alpha_{\mathrm{heterog}}K_{\mathrm{heterog}}>\alpha_{\mathrm{homog}}K_{\mathrm{homog}}$, the heterogeneous design enjoys a strictly higher information-recovery guarantee, since by Lemma~\ref{lem:monotone_alphaK} the function $1-e^{-t}$ is strictly increasing in $t$.
\end{corollary}

\begin{proof}
Apply Corollary~\ref{cor:lb} to each design to obtain \eqref{eq:lb_heterog_app} and \eqref{eq:lb_homog_app}.
Since $\alpha_{\mathrm{heterog}}K_{\mathrm{heterog}}>\alpha_{\mathrm{homog}}K_{\mathrm{homog}}$ and $f(t)=1-e^{-t}$ is strictly increasing (Lemma~\ref{lem:monotone_alphaK}), the lower bound for the heterogeneous design is strictly larger than that for the homogeneous design.
\end{proof}

\subsection{Properties of the Effective Channel Count $K^*$}
\label{app:kstar}

This section proves basic properties of the label-free proxy $K^*$ used in the main text
(Section~\ref{sec:Kstar}). We restate the definition for completeness.

\paragraph{Setup.}
Given $n$ outputs, let $\hat{\mathbf{z}}_i := \mathrm{Emb}(Z_i)/\|\mathrm{Emb}(Z_i)\|_2 \in \mathbb{R}^d$ be the normalized embeddings, and let $M\in\mathbb{R}^{n\times d}$ be the embedding matrix whose $i$-th row is $\hat{\mathbf{z}}_i^\top$.
Define the cosine-similarity Gram matrix $G_{ij} := \langle \hat{\mathbf{z}}_i, \hat{\mathbf{z}}_j \rangle$ (equivalently $G=MM^\top$) and its trace-normalization
\begin{equation}
\rho := \frac{G}{\mathrm{Tr}(G)}.
\end{equation}
Let $\{\lambda_j\}_{j=1}^n$ be the eigenvalues of $\rho$.
The von Neumann entropy is
\begin{equation}
H(\rho):=-\sum_{j=1}^n \lambda_j \log_2 \lambda_j,
\end{equation}
and the effective channel count is $K^*:=2^{H(\rho)}$.

\begin{proposition}[Basic Properties of $K^*$]
\label{prop:kstar-properties}
For any nonzero embedding matrix $M$,
\begin{enumerate}[leftmargin=1.2em]
\item $1\le K^*\le n$.
\item $K^*=1$ iff $\rho$ is rank-$1$ (all embeddings are collinear up to scaling).
\item $K^*=n$ iff $\rho=\frac{1}{n}I_n$ (embeddings are orthogonal with equal norm).
\item $K^*$ is continuous in $M$ (when $\mathrm{Tr}(G)>0$) and invariant to permutation of outputs.
\end{enumerate}
\end{proposition}

\begin{proof}
\textbf{(i) Bounds.}
Entropy satisfies $0\le H(\rho)\le \log_2 n$, hence
$1=2^0\le 2^{H(\rho)}\le 2^{\log_2 n}=n$.

\textbf{(ii) $K^*=1$.}
$H(\rho)=0$ iff the spectrum is $(1,0,\ldots,0)$, which holds iff $\rho$ is rank-$1$.
This corresponds to all rows of $M$ being collinear, i.e., all embeddings are identical up to scaling.

\textbf{(iii) $K^*=n$.}
$H(\rho)=\log_2 n$ iff $\lambda_j=1/n$ for all $j$, which occurs when $\rho=\frac{1}{n}I_n$.
This corresponds to $G$ being proportional to the identity, i.e., embeddings are orthogonal with equal norm.

\textbf{(iv) Continuity and permutation invariance.}
The map $M\mapsto G=MM^\top$ is continuous.
Normalization by $\mathrm{Tr}(G)$ is continuous when $\mathrm{Tr}(G)>0$.
Eigenvalues of a symmetric matrix vary continuously with the entries, and entropy is continuous on the simplex.
Permutation of outputs corresponds to $G\mapsto PGP^\top$ for a permutation matrix $P$, which preserves eigenvalues.
\end{proof}

\section{Supplementary Experiments}

\subsection{Closed-Source Model Experiments}
\label{sec:closed_source}

We extend our analysis to closed-source models (gpt-4.1-mini, gpt-5-mini) on the Formal Logic benchmark to test whether the heterogeneity advantage generalizes across model families.
Table~\ref{tab:closed_source_comparison} compares closed- and open-source models under homogeneous (Base) and heterogeneous (Heterog) configurations at $N\!=\!2$ and $N\!=\!16$.

\begin{table}[t]
\centering
\caption{Closed-source vs.\ Open-source on Formal Logic (\%). $\Delta_\text{Het}$: average Heterog gain over Base. $\Delta_N$: accuracy change from $N\!=\!2$ to $N\!=\!16$ under the Heterog configuration.}
\label{tab:closed_source_comparison}
\footnotesize
\setlength{\tabcolsep}{7pt} %
\begin{tabular}{llccccccc}
\toprule
\multirow{2}{*}{\textbf{Model}} & \multirow{2}{*}{\textbf{Method}} & \multicolumn{2}{c}{\textbf{Base}} & \multicolumn{2}{c}{\textbf{Heterog}} & \multirow{2}{*}{\textbf{$\Delta_\text{Het}$}} & \multirow{2}{*}{\textbf{$\Delta_N$}} \\
\cmidrule(lr){3-4} \cmidrule(lr){5-6}
 &  & N=2 & N=16 & N=2 & N=16 &  &  \\
\midrule
\multicolumn{8}{l}{\textit{Closed-source}} \\
\midrule
\multirow{2}{*}{gpt-4.1-mini}
 & vote   & 46.83 & 48.41 & 55.56 & 52.38 & +6.35  & $-$3.18 \\
 & debate & 46.03 & 39.68 & 50.79 & 42.86 & +3.97  & $-$7.93 \\
\midrule
\multirow{2}{*}{gpt-5-mini}
 & vote   & 0.00  & 6.35  & 35.71 & 49.21 & \textbf{+39.29} & +13.50 \\
 & debate & 0.79  & 6.35  & 34.92 & 54.76 & \textbf{+41.27} & +19.84 \\
\midrule
\multicolumn{8}{l}{\textit{Open-source}} \\
\midrule
\multirow{2}{*}{Qwen-2.5-7B}
 & vote   & 38.10 & 44.44 & 45.24 & 50.00 & +6.35  & +4.76 \\
 & debate & 30.95 & 34.13 & 25.40 & 38.10 & $-$0.79 & +12.70 \\
\midrule
\multirow{2}{*}{Llama-3.1-8B}
 & vote   & 45.24 & 42.86 & 44.44 & 54.76 & +5.55  & +10.32 \\
 & debate & 24.60 & 31.75 & 35.71 & 39.68 & +9.52  & +3.97 \\
\midrule
\multirow{2}{*}{Mistral-7B}
 & vote   & 35.71 & 42.06 & 34.92 & 41.27 & $-$0.79 & +6.35 \\
 & debate & 34.92 & 44.44 & 39.68 & 46.83 & +3.58  & +7.15 \\
\bottomrule
\end{tabular}%
\end{table}

\paragraph{Key findings.}
The results confirm that the heterogeneity advantage generalizes to closed-source models, while revealing that its magnitude and scaling behavior vary across model families.

\begin{itemize}[nosep,leftmargin=*]
    \item \textbf{Heterogeneity consistently improves over homogeneous baselines.}
    All five models exhibit positive $\Delta_\text{Het}$ in at least one interaction mechanism, confirming that the advantage is not specific to open-source settings.

    \item \textbf{Models with weaker homogeneous baselines benefit more from heterogeneity.}
    gpt-5-mini achieves near-zero accuracy under homogeneous settings (0--6\%) but reaches 35--55\% with heterogeneous prompting ($\Delta_\text{Het}$ of +39--41\%).
    In contrast, gpt-4.1-mini and the open-source models, which already achieve 25--46\% under homogeneous settings, show more modest gains (+3--10\%).

    \item \textbf{Scaling trends diverge across model families.}
    Open-source models exhibit positive scaling ($\Delta_N > 0$) under both configurations.
    gpt-4.1-mini, however, shows \emph{negative} scaling in debate: accuracy drops from 50.79\% to 42.86\% ($\Delta_N = -7.93$) even under heterogeneous settings, indicating that adding more agents can hurt when the base model is already strong.
    gpt-5-mini shows the opposite pattern: under heterogeneous settings it benefits substantially from more agents ($\Delta_N = +19.84$ for Debate), whereas its homogeneous scaling remains near-flat.
\end{itemize}

\subsection{Robustness to Embedding Model Choice}
\label{sec:embedding_robustness}

A potential concern is whether our effective channel metric $K^*$ depends critically on the choice of embedding model. To address this, we recompute $K^*$ using a different embedding model, gte-Qwen2-1.5B-instruct (1536 dimensions), and compare the results against our primary model NV-Embed-v2 (4096 dimensions). We conduct this comparison across seven datasets (ARC, Formal Logic, GSM8K, HellaSwag, Pro Medicine, TruthfulQA, WinoGrande), varying agent counts $N \in \{2, 4, 8, 12, 16\}$ and interaction mechanisms (Vote and Debate).

Since embedding dimensionality affects absolute $K^*$ values, direct comparison of raw values across models is not meaningful. Instead, we assess robustness by measuring whether the two embeddings agree on \emph{relative ordering}: within each (configuration type, dataset) pair, do both embeddings rank different (method, $N$) combinations consistently? Across all matched pairs, we observe an average Spearman correlation of $\rho = 0.91$, with over 95\% of pairs showing $\rho > 0.5$. This indicates that both embeddings consistently identify which experimental settings produce more diverse outputs, even though their absolute scales differ.

Furthermore, both embeddings yield $K^*$ metrics that positively correlate with task accuracy (NV-Embed-v2: $r = 0.40$; gte-Qwen2: $r = 0.23$), confirming that our core finding, diversity predicts performance, is not an artifact of a particular embedding choice. We use NV-Embed-v2 in the main experiments as it achieves stronger predictive power.

\subsection{Is $K^*$ More Than a Proxy for Scale and Configuration?}
\label{sec:kstar_incremental}
Since $K^*$ is computed from agent outputs whose diversity naturally varies with agent count $N$ and configuration type, a key question is whether $K^*$ captures information \emph{beyond} these design variables, or merely serves as a redundant proxy for them.
To disentangle this, we fit a baseline regression that predicts task accuracy from $N$ and configuration labels alone, then measure the incremental variance explained ($\Delta R^2$) when $K^*$ or its components are added.

\begin{table*}[h]
\centering
\caption{\textbf{Incremental Explanatory Power of Effective Channels.}
The baseline model using only agent count ($N$) and configuration labels explains little variance ($R^2 = 0.062$). Adding $K^*$ substantially improves fit ($\Delta R^2 = +0.147$), and conditioning on answer correctness ($K^*_c$) yields the largest gain ($\Delta R^2 = +0.331$), while $K^*_w$ contributes negligibly.}

\label{tab:exp1-incremental-explanatory}
\footnotesize
\setlength{\tabcolsep}{5pt}
\begin{tabular}{l|cccc}
\toprule
\textbf{Model} & \textbf{$R^2$} & \textbf{Adj.\ $R^2$} & \textbf{$\Delta R^2$} & \textbf{AIC} \\
\midrule
Baseline ($N$ + Config) & 0.062 & 0.044 & -- & 1806.6 \\
Baseline + $K^*$        & 0.209 & 0.190 & +0.147 & 1771.1 \\
Baseline + $K^*_c$      & 0.393 & 0.378 & +0.331 & 1713.0 \\
Baseline + $K^*_c + K^*_w$ & 0.396 & 0.379 & +0.334 & 1713.8 \\
Baseline + $K^*_c/K^*_w$   & 0.325 & 0.309 & +0.263 & 1736.4 \\
\bottomrule
\end{tabular}
\end{table*}

Table~\ref{tab:exp1-incremental-explanatory} reveals three findings.
First, the baseline model with only $N$ and configuration labels achieves $R^2 = 0.062$, confirming that scale and configuration alone are poor predictors of MAS performance.
Second, adding $K^*$ raises $\Delta R^2$ by $+0.147$, demonstrating that it captures structural information about output diversity that is not reducible to agent count or configuration choice.
Third, and most importantly, replacing $K^*$ with its correctness-conditioned component $K^*_c$ more than doubles the incremental gain ($\Delta R^2 = +0.331$), while further adding $K^*_w$ yields negligible improvement ($\Delta R^2$: $+0.331 \to +0.334$).
This asymmetry directly supports our central thesis: what drives MAS performance is not output diversity in general, but specifically the diversity of \emph{correct} reasoning paths. Increasing the number of distinct ways agents arrive at the right answer is far more predictive than total channel count or the diversity of incorrect responses.

\subsection{Sanity Checks: Are $K^*$--Performance Relations Accidental?}
\label{sec:kstar_sanity}

We further test whether the observed relationship between effective channels
and performance could arise by chance.
To this end, we conduct permutation-based randomization tests that preserve
the marginal distribution of accuracy while destroying any structural
association with $K^*$.

\begin{table}[h]
\centering
\caption{\textbf{Permutation Sanity Check (1000 shuffles).}
Observed correlations between effective-channel metrics and accuracy lie far
outside the null distribution, confirming that the relationship is not due to chance.}
\label{tab:exp3-permutation}
\setlength{\tabcolsep}{6pt}
\begin{tabular}{lccc}
\toprule
\textbf{Metric} & \textbf{Observed $r$} & \textbf{$z$-score} & \textbf{$p$} \\
\midrule
$K^*$            & 0.388 & 5.87 & $<$0.001 \\
$K^*_c$          & 0.535 & 7.75 & $<$0.001 \\
$K^*_c/K^*_w$    & 0.503 & 7.23 & $<$0.001 \\
\bottomrule
\end{tabular}
\end{table}

As shown in Table~\ref{tab:exp3-permutation}, all effective-channel metrics exhibit
$z$-scores well above 5 under permutation testing, with $p<10^{-3}$.
This rules out the possibility that the observed correlations arise from random
alignment or dataset-specific artifacts.
Notably, $K^*_c$ again yields the strongest signal, reinforcing the interpretation
that correct-path diversity is the dominant driver of multi-agent performance.

\begin{table*}[t]
\centering
\caption{Model Ablation on Formal Logic: Impact of Heterogeneity from $N=2$ to $N=16$}
\label{tab:logic_full_ablation}
\small
\begin{tabular}{lcccccccc}
\toprule
\multirow{2}{*}{\textbf{Base Model}} & \multirow{2}{*}{\textbf{Agents ($N$)}} & \multicolumn{3}{c}{\textbf{Vote (Round 0)}} & \multicolumn{3}{c}{\textbf{Debate (Final)}} \\
\cmidrule(lr){3-5} \cmidrule(lr){6-8}
 &  & \textbf{Homog} & \textbf{Heterog} & \textbf{$\Delta_{H-M}$} & \textbf{Homog} & \textbf{Heterog} & \textbf{$\Delta_{H-M}$} \\
\midrule
\multirow{5}{*}{Qwen-2.5-7B} 
 & 2  & 38.10\% & 45.24\% & +7.14\% & 30.95\% & 25.40\% & -5.55\% \\
 & 4  & 42.06\% & 53.97\% & +11.91\% & 30.16\% & 34.92\% & +4.76\% \\
 & 8  & 43.65\% & 50.00\% & +6.35\% & 28.57\% & 38.10\% & +9.53\% \\
 & 12 & 44.44\% & 52.38\% & +7.94\% & 31.75\% & 35.71\% & +3.96\% \\
 & 16 & 44.44\% & 50.00\% & +5.56\% & 34.13\% & 38.10\% & +3.97\% \\
\midrule
\multirow{5}{*}{Llama-3.1-8B} 
 & 2  & 45.24\% & 44.44\% & -0.80\% & 24.60\% & 35.71\% & +11.11\% \\
 & 4  & 42.86\% & 53.97\% & +11.11\% & 23.02\% & 24.60\% & +1.58\% \\
 & 8  & 41.27\% & 52.38\% & +11.11\% & 27.78\% & 35.71\% & +7.93\% \\
 & 12 & 43.65\% & 53.97\% & +10.32\% & 30.95\% & 38.89\% & +7.94\% \\
 & 16 & 42.86\% & 54.76\% & +11.90\% & 31.75\% & 39.68\% & +7.93\% \\
\midrule
\multirow{5}{*}{Mistral-7B} 
 & 2  & 35.71\% & 34.92\% & -0.79\% & 34.92\% & 39.68\% & +4.76\% \\
 & 4  & 34.92\% & 36.51\% & +1.59\% & 35.71\% & 44.44\% & +8.73\% \\
 & 8  & 32.54\% & 37.30\% & +4.76\% & 40.48\% & 38.89\% & -1.59\% \\
 & 12 & 38.89\% & 38.10\% & -0.79\% & 42.06\% & 42.86\% & +0.80\% \\
 & 16 & 42.06\% & 41.27\% & -0.79\% & 44.44\% & 46.83\% & +2.39\% \\
 \midrule
 \multirow{5}{*}{MIX} 
 & 2  & 45.24\% & 48.41\% & +3.17\% & 34.13\% & 38.89\% & +4.76\% \\
 & 4  & 47.62\% & 52.38\% & +4.76\% & 42.86\% & 53.17\% & +10.31\% \\
 & 8  & 47.62\% & 55.56\% & +7.94\% & 49.21\% & 53.17\% & +3.96\% \\
 & 12 & 48.41\% & 57.94\% & +9.53\% & 48.41\% & 54.76\% & +6.35\% \\
 & 16 & 50.00\% & 53.97\% & +3.97\% & 43.65\% & 51.59\% & +7.94\% \\
\bottomrule
\end{tabular}
\end{table*}

\subsection{Case Study: Heterogeneity Effects Across Models and Workflows}

Table~\ref{tab:logic_full_ablation} reports a comprehensive ablation study on the
Formal Logic benchmark, varying base models, agent counts ($N=2$--$16$), and
interaction mechanisms.
Across nearly all settings, heterogeneous configurations outperform homogeneous
ones, often by substantial margins.
Importantly, these gains do not arise from scaling alone.
For example, in both Vote and Debate, increasing $N$ beyond moderate values
frequently yields diminishing or unstable returns in homogeneous settings,
while heterogeneous systems maintain consistent improvements.
This pattern holds across all three base models and their mixture, indicating that
the benefit of heterogeneity is robust to model choice and interaction protocol.

\begin{table*}[h]
\centering
\caption{Formal Logic: Synergy of Model Mixing (MIX vs. Best Single Model)}
\label{tab:logic_mix_synergy}
\small
\begin{tabular}{lccccc}
\toprule
\textbf{Agents ($N$)} & \textbf{Best Single (Heterog)} & \textbf{MIX (Heterog)} & \textbf{$\Delta_{\text{MIX vs. Best}}$} & \textbf{MIX (Homog)} & \textbf{$\Delta_{\text{H-M (MIX)}}$} \\
\midrule
2  & 39.68\%  & 38.89\% & -0.79\% & 34.13\% & +4.76\% \\
4  & 44.44\%  & 53.17\% & \textbf{+8.73\%} & 42.86\% & +10.31\% \\
8  & 38.89\%  & 53.17\% & \textbf{+14.28\%} & 49.21\% & +3.96\% \\
12 & 42.86\%  & 54.76\% & \textbf{+11.90\%} & 48.41\% & +6.35\% \\
16 & 46.83\%  & 51.59\% & \textbf{+4.76\%} & 43.65\% & +7.94\% \\
\bottomrule
\end{tabular}
\end{table*}

Table~\ref{tab:logic_mix_synergy} isolates the effect of model mixing by comparing
a heterogeneous mixture (MIX) against the best-performing single model under
the same agent count.
At $N \geq 4$, MIX consistently outperforms the strongest individual model by
large margins, reaching up to +14.28\% absolute accuracy at $N=8$.

Crucially, these gains cannot be explained by model selection alone.
Even when the best single model is used with heterogeneous prompting, the MIX
configuration achieves higher performance, demonstrating genuine synergy across
models rather than simple averaging or dominance effects.

\end{document}